\documentclass{article}

\usepackage{microtype}
\usepackage[margin=1in]{geometry}
\usepackage{graphicx}
\usepackage{subfigure}
\usepackage{booktabs} 
\usepackage{natbib}
\usepackage{hyperref}
\usepackage{amsmath,mathrsfs}
\usepackage{dsfont}
\usepackage{amssymb}
\usepackage{amsthm}
\usepackage{bm}
\usepackage{verbatim}
\usepackage{mathtools}
\mathtoolsset{showonlyrefs}

\newtheorem{theorem}{Theorem}

\newtheorem{lemma}{Lemma}
\newtheorem{proposition}{Proposition}
\newtheorem{remark}[theorem]{Remark}

\newtheorem{assumption}{Assumption}

\usepackage{algorithm}
\usepackage{algorithmic}
\newcommand{\nc}[1]{\textcolor{red}{NC: #1}}

\newcommand{\R}{\mathbb{R}}

\newcommand{\E}{\mathbb{E}}

\newcommand{\B}{\mathcal{B}}
\newcommand{\I}{\mathcal{I}}
\newcommand{\Prob}{\mathbb{P}}
\newcommand{\1}{\mathds{1}}
\DeclareMathOperator*{\argmin}{\mathrm{argmin}}
\DeclareMathOperator*{\argmax}{\mathrm{argmax}}

\usepackage{lineno}
\usepackage{authblk}

\makeatletter
\newcommand{\printfnsymbol}[1]{\textsuperscript{\@fnsymbol{#1}}}
\makeatother

\begin{document}

\title{Regime Switching Bandits}


\author{Xiang Zhou\thanks{Equal contribution} 
\thanks{Department of Systems Engineering and Engineering Management; The Chinese University of Hong Kong; zhouxng@se.cuhk.edu.hk} \quad Yi Xiong\printfnsymbol{1}\thanks{Department of Systems Engineering and Engineering Management; The Chinese University of Hong Kong; yxiong@se.cuhk.edu.hk} \quad Ningyuan Chen\thanks{The Rotman School of Management; University of Toronto; ningyuan.chen@utoronto.ca} \quad Xuefeng Gao\thanks{Department of Systems Engineering and Engineering Management; The Chinese University of Hong Kong; xfgao@se.cuhk.edu.hk}}

\maketitle
\renewcommand\thefootnote{\textparagraph}
\footnotetext{Equal contribution}

\begin{abstract}
We study a multi-armed bandit problem where the rewards exhibit regime switching. Specifically, the distributions of the random rewards generated
from all arms are modulated by a common underlying state modeled as a finite-state Markov chain. The agent does not observe the underlying state and has to learn the transition matrix and the reward distributions. We propose a learning algorithm for this problem, building on spectral method-of-moments estimations for hidden Markov models, belief error control in partially observable Markov decision processes and upper-confidence-bound methods for online learning. We also establish an upper bound $O(T^{2/3}\sqrt{\log T})$ for the proposed learning algorithm where $T$ is the learning horizon. Finally,
we conduct proof-of-concept experiments to illustrate
the performance of the learning algorithm.
\end{abstract}

	\section{Introduction}
	The multi-armed bandit (MAB) problem is a popular model for sequential decision making with unknown information: the decision maker makes decisions repeatedly among $I$ different options, or arms.
	After each decision she receives a random reward having an \textit{unknown} probability distribution that depends on the chosen arm.
	The objective is to maximize the expected total reward over a finite horizon of $T$ periods.
	The MAB problem has been extensively studied in various fields and applications including Internet advertising, dynamic pricing, recommender systems, clinical trials and medicine. See, e.g., \cite{Bouneffouf2019, Bubeck2012, Slivkins2019}.
	In the classical MAB problem, it is typically assumed that the random reward of each arm is i.i.d. (independently and identically distributed) over time and
	independent of the rewards from other arms.
	However, these assumptions do not necessarily hold in practice \cite{Besbes2014}.
	To address the drawback, a growing body of literature studies MAB problems with non-stationary rewards to capture temporal changes in the reward distributions in applications, see e.g. \cite{Besbes2019, Cheung2018, Garivier2011}. 

	In this paper, we study a non-stationary MAB model with Markovian regime-switching rewards.
	We assume that the random rewards associated with all the arms are modulated by a common \emph{unobserved} state (or regime) $\{M_t\}$ modeled as a finite-state Markov chain.
	Given $M_t=m$, the reward of arm $i$ is i.i.d., whose distribution is denoted $Q(\cdot|m,i)$.
	Such structural change of the environment is usually referred to as regime switching in finance \cite{mamon2007hidden}.
	The agent doesn't observe or control the underlying state $M_t$, and has to learn the transition probability matrix $P$ of $\{M_t\}$ as well as the distribution of reward of each arm $Q(\cdot|m,i)$, based on the observed historical rewards. The goal of the agent is to design a learning policy that decides which arm to pull in each period to minimize the expected regret over $T$ periods.

	The regime-switching models are widely used in various industries. For example, in finance, the sudden changes of market environments are usually modeled as hidden Markov chains. In revenue management and marketing, a firm may face a shift in consumer demand due to undetected changes in sentiment or competition. In such cases, when the agents (traders or firms) take actions (trading financial assets with different strategies or setting prices), they need to learn the reward and the underlying state at the same time. Our setup is designed to tackle such problems.

	\textbf{Our Contribution.}
	Our study features novel designs in three regards.
	In terms of problem formulation, online learning with unobserved states has attracted some attention recently \cite{Azizzadenesheli2016,Fiez2019}.
	We consider the strongest oracle among the studies, who knows $P$ and $Q(\cdot|m,i)$, but doesn't observe the hidden state $M_t$.
	The oracle thus faces a partially observable Markov decision process (POMDP) \cite{Krish2016}.
	By reformulating it as a Markov decision process (MDP) with a continuous belief space (i.e., a distribution over hidden states), the oracle then solves the optimal policy using the Bellman equation.
	Having sublinear regret benchmarked against the strong oracle, our algorithm has better theoretical performance than others with weaker oracles such as the best fixed arm \cite{Fiez2019} or memoryless policies (the action only depends on the current observation) \cite{Azizzadenesheli2016}.

	In terms of algorithmic design, we propose a learning algorithm (see Algorithm \ref{alg:SEEU}) with two key ingredients.
	First, it builds on the recent advance on the estimation of the parameters of hidden Markov models (HMMs) using spectral method-of-moments methods, which involve the spectral decomposition of certain low-order multivariate moments computed from data \cite{Anandkumar2012, Anandkumar2014, Azizzadenesheli2016}.
	It benefits from the theoretical finite-sample bound of spectral estimators,
	while the finite-sample guarantees of other alternatives such as maximum likelihood estimators remain an open problem \cite{Lehericy2019}.
	Second, it builds on the well-known
	``upper confidence bound'' (UCB) method in reinforcement learning \cite{Auer2007, Jaksch2010}. To deal with the belief of the hidden state which is subject to the estimation error,
	we divide the horizon into nested exploration and exploitation phases.
	Note that this is a unique challenge in our formulation as the oracle uses the optimal policy of the POMDP (belief MDP).
	We use spectral estimators in the exploration phase
	to gauge the estimation error of $P$ and $Q(\cdot|m,i)$.
	We use the UCB method to control the regret in the exploitation phase.
	Different from other learning problems, we re-calibrate the belief at the beginning of each exploitation phase based on the parameters estimated in the most recent exploration phase using previous exploratory samples.


	In terms of technical analysis, we establish a regret bound of $O(T^{2/3} \sqrt{\log(T)})$ for our proposed learning algorithm where $T$ is the learning horizon. 
	Our analysis is novel in two aspects.
	First, we need to control of the error of the belief state, which itself is not directly observed.
	This is in stark contrast to learning MDPs \citep{Jaksch2010, Ortner2012} with observed states.
	Note that the estimation error of the parameters affect the updating of the belief in every previous period.
	Hence one would expect that the error in the belief state might accumulate over time and explode even if the estimation of the parameters is quite accurate.
	We show that this is not the case and as a result the regret attributed to the belief error is well controlled.
	Second, we provide an explicit bound for the span of the bias function (also referred as the relative value function) for the POMDP or the belief MDP which has a continuous state space. Such a bound is often critical in the regret analysis of undiscounted reinforcement learning of continuous MDP, but it is either \textit{taken as an assumption} \citep{Qian2018} or proved under
	H\" older continuity assumptions that do not hold for the belief transitions in our setting \citep{Ortner2012, Lakshmanan2015}. We provide an explicit bound on the bias span by using an approach based on the classical vanishing discount technique and \citep{Hinderer2005} which provides general tools for proving Lipschitz continuity of value functions in finite-horizon discounted MDPs with general state spaces.  

	\textbf{Related Work.}
	Studies on non-stationary/switching MAB investigate the problem when the rewards of all arms may change over time \cite{auer2002nonstochastic, Garivier2011, Besbes2014,auer2019adaptively}.
	Our formulation can be regarded as non-stationary MAB with a special structure.
	They consider an even stronger oracle than ours, the best arm in each period.
	However, the total number of changes or the changing budget have to be sublinear in $T$ to achieve sublinear regret.
	In our formulation, the total number of transitions of the underlying state is linear in $T$, and the algorithms proposed in these papers fail even considering our oracle (See Section~\ref{sec:numerical}).
	Other studies focus on linear changing budget with some structure such as seasonality \cite{chen2020learning}, which is not present in this paper.

	Our work is also related to the restless Markov bandit problem
	\cite{Guha2010, Ortner2014, Slivkins2008} in which the state of each arm evolves according to independent Markov chains. In contrast, our regime-switching MAB model assumes a \textit{common} underlying Markov chain so that the rewards of all arms are correlated, and the underlying state is unobservable to the agent.
	In addition, our work is related to MAB studies where rewards of all arms depend on a common unknown parameter or a latent random variable, see, e.g., \cite{Atan2018, Gupta2018, Lattimore2014, Maillard2014, Mersereau2009}. Our model differs from them in that the common latent state variable follows a dynamic stochastic process which introduces difficulties in algorithm design and analysis.

	Two papers have similar settings to ours.
	\cite{Fiez2019} studies MAB problems whose rewards are modulated by an unobserved Markov chain and the transition matrix may depend on the action.
	However, their oracle is the best fixed arm when defining the regret, which is much weaker than the optimal policy of the POMDP (the performance gap is linear between the two oracles).
	Therefore, their algorithm is expected to have linear regret when using our oracle. \cite{Azizzadenesheli2016} also proposes to use spectral estimators to learn POMDPs.
	Their oracle is the optimal memoryless policy, i.e., a policy that only depends on the current reward observation instead of using all historical observations to form the belief of the underlying state (the performance gap is generally linear in $T$ between their oracle and ours).
	Such a policy set allows them to circumvent the introduction of the belief entirely.
	As a result, the dynamics under the memoryless policy can be viewed as a finite-state (modified) HMM and spectral estimators can be applied.
	Instead, because we consider the optimal belief-based policy, such reduction is not available.
	Our algorithm and regret analysis hinge on the interaction between the estimation of the belief and the spectral estimators.
	Our algorithm needs separate exploration to apply the spectral methods and uses the belief-based policy for exploitation. For the regret analysis, unlike \cite{Azizzadenesheli2016}, we have to carefully control the belief error and bound the span of the bias function from the optimistic belief MDP in each episode.
	The comparison of our study with related papers is summarized in Table~\ref{tab:literature}.  Note that we only present the regret in terms of $T$.
	\begin{table}[h!]
            \centering
		\begin{tabular}{llll}
			\toprule
			Papers & Oracle & Changing Budget & Regret \\
			\midrule
			\cite{auer2002nonstochastic}&Best fixed action&Linear & $\tilde O(\sqrt{T})$\\
			\cite{Garivier2011}       &Best action in each period& Finite & $\tilde{O}(\sqrt{T})$\\
			\cite{Besbes2014}       &Best action in each period& Sublinear & $\tilde{O}(T^{2/3})$ \\
			\cite{Azizzadenesheli2016}& Optimal memoryless policy& Linear&  $\tilde O(\sqrt{T})$\\
			This paper & Optimal POMDP policy& Linear&  $\tilde O(T^{2/3})$\\
			\bottomrule
		\end{tabular}
		\caption{Comparison of our study with some related literature.}
		\label{tab:literature}
	\end{table}

	\section{MAB with Markovian Regime Switching}
	\subsection{Problem Formulation}
	Consider the MAB problem with arms $\mathcal{I}\coloneqq\left\{1,\dots,I\right\}$.
	There is a Markov chain $\{M_t\}$ with states $\mathcal{M}\coloneqq\{1,2,\dots,M\}$ and transition probability matrix $P\in \R^{M\times M}$.
	In period $t=1,2,\dots$, if the state of the Markov chain $M_t=m$ and the agent chooses arm $I_t= i$, then the reward in that period is $R_t$ with discrete finite support, and its distribution is denoted by
	$Q(\cdot|m,i)\coloneqq \Prob (R_t\in \cdot|M_t=m,I_t=i)$,
	with $\mu_{m,i}\coloneqq \E[R_t|M_t=m,I_t=i]$.
	We use $\bm \mu\coloneqq(\mu_{m,i})\in \R^{M\times I}$ to denote the mean reward matrix.
	The agent knows $M$ and $I$, but has no knowledge about the underlying state $M_t$ (also referred to as the regime), the transition matrix $P$ or the reward distribution $Q(\cdot|m,i)$.
	The goal is to design a learning policy that is adapted to the filtration generated by the observed rewards to decide which arm to pull in each period to maximize the expected cumulative reward over $T$ periods where $T$ is unknown in advance.

	If an oracle knows $P$, $Q(\cdot|m,i)$ and the underlying state $M_t$, then the problem becomes trivial as s/he would select $I^{\ast}_t=\argmax_{i\in\mathcal{I}} \mu_{M_t,i}$ in period $t$.
	If we benchmark a learning policy against the oracle, then the regret 
	must be linear in $T$,
	because the oracle always observes $M_t$ while the agent cannot predict the transition based on the history.
	Whenever a transition occurs, there is non-vanishing regret incurred.
	Since the number of transitions during $[0,T]$ is linear in $T$,
	the total regret is of the same order.
	Since comparing to the oracle knowing $M_t$ is uninformative,
	we consider a weaker oracle who knows $P$, $Q(\cdot|m,i)$, \emph{but not} $M_t$.
	In this case, the oracle solves a POMDP since the states $M_t$ are unobservable.
	The total expected reward of the POMDP can be shown to scale linearly in $T$, and  asymptotically the reward per period converges to a constant denoted by $\rho^*$. See Section~\ref{subsec:belief MDP}.

	For a learning policy $\pi$,
	we denote by $R_t^{\pi}$ the reward received under the learning policy $\pi$ (which does not know $P, Q(\cdot|m,i)$ initially) in period $t$.
	We follow the literature (see, e.g., \cite{Jaksch2010,  Ortner2014, Agrawal17}) and define its total regret after $T$ periods by
	\begin{align}\label{def:reg}
	\mathcal{R}_T\coloneqq T\rho^* - \sum_{t=1}^T R_t^{\pi}.
	\end{align}
	Our goal is to design a learning algorithm with theoretical guarantees including high probability and expectation bounds (sublinear in $T$) on the total regret.

	Without loss of generality we consider Bernoulli rewards with mean $\mu_{m,i} \in (0,1)$ for all $m, i$.
	Hence $\mu_{m,i}$ characterizes the distribution $Q(\cdot|m,i)$. Our analysis holds generally for random rewards with discrete finite support.
	In addition, we impose the following assumptions.
	\begin{assumption}\label{assum:Markov Chain}
		The transition matrix $P$ of the Markov chain $\{M_t\}$  is invertible.
	\end{assumption}
	\begin{assumption}\label{assum:reward_matrix2}
		The mean reward matrix $\bm \mu= (\mu_{m,i})$ has full row rank.
	\end{assumption}

	\begin{assumption}\label{assum:reward_matrix3}
		The smallest element of the transition matrix $\epsilon\coloneqq \min_{i,j \in \mathcal{M}} P_{ij}>0$.
	\end{assumption}
	The first two assumptions are required for the finite-sample guarantee of spectral estimators for HMMs \cite{Anandkumar2012,Anandkumar2014}.
	The third assumption is required for controlling the belief error of the hidden states by the state-of-art methods; See \cite{DeCastro2017} and our Proposition~\ref{prop:lip_bt} in Section~\ref{sec:proof1}. We next reformulate the POMDP as a belief MDP.



	\subsection{Reduction of POMDP to Belief MDP}\label{subsec:belief MDP}
	To present our learning algorithm and analyze the regret, we first investigate the POMDP problem faced by the oracle where parameters $P$, $\bm \mu$ (equivalently $Q$ for Bernoulli rewards) are known with unobserved states $M_t$.
	Based on the historical observed history, the oracle forms a belief of the underlying state.
	The belief can be encoded by a $M$-dimension vector $b_t =\left(b_t(1),\dots,b_t(M)\right)\in \mathcal B$:
	\begin{align}\label{def:bt}
	b_t(m)\coloneqq \mathbb{P}(M_t=m|I_1, \cdots, I_{t-1}, R_1, \cdots, R_{t-1}),
	\end{align}
	where $\mathcal{B} \coloneqq \left\{b\in \mathbb{R}_+^{M}: \sum_{m=1}^{M}b(m)=1 \right\}.$
	It is well known that the POMDP of the oracle can be seen as a MDP built on a continuous state space $\mathcal{B}$, see \cite{Krish2016}.

	We next introduce a few notations that facilitate the analysis.
	For notation simplicity, we write $c(m,i)\coloneqq \mu_{m,i}$.
	Given belief $b\in\mathcal B$,
	the expected reward of arm $i$ is $\bar{c}(b,i)\coloneqq \sum_{m=1}^{M}c(m,i)b(m).$
	In period $t$, given belief $b_t=b, I_t=i, R_t=r $, by Bayes' theorem, the belief $b_{t+1}$ is updated by $b_{t+1}=H_{\bm{\mu},P}(b,i,r)$,
	where the $m$-th entry is
	$$b_{t+1}(m)=\frac{\sum \limits_{m'}P(m',m)\cdot (\bm{\mu}_{m',i})^r(1-\bm{\mu}_{m',i})^{1-r}\cdot b_t(m')}{\sum \limits_{m''}(\bm{\mu}_{m'',i})^r(1-\bm{\mu}_{m'',i})^{1-r}\cdot b_t(m'')}.$$
	It is obvious that the forward function $H$ depends on the transition matrix $P$ and the reward matrix $\bm \mu$ . We can also define the transition probability of the belief state conditional on the arm pulled: $\bar{T}(\cdot|b,i):=\Prob(b_{t+1} \in \cdot|b,i)$, where $b_{t+1}$ is random due to the random reward.

	The long-run average reward of the infinite-horizon belief MDP following policy $\pi$  given the initial belief $b$ can be written as
	$\rho^\pi_b:= \limsup_{T \to \infty}\frac{1}{T}  \E [\sum_{t=1}^{T} R_t^{\pi} |b_1 =b]$.
	The optimal policy maximizes $\rho^\pi_b$ for a given $b$.
	We use $\rho^{*}\coloneqq \sup_{b}\sup_{\pi}\rho_b^\pi$
	to denote the optimal long-run average reward.
	Under this belief MDP formulation, for all $b\in \B$, the Bellman equation states that
	\begin{align}\label{equ:Bellman-thm}
	\rho^*+v(b)=\max_{i\in\I}\left[\bar{c}(b,i)+\int_{\B}\bar{T}(db'|b,i)v(b')\right],
	\end{align}
	where $v:\B\mapsto\R$ is the bias function.
	It can be shown (see the proof of Proposition~\ref{prop:span-uni-bound} in Appendix) that under our assumptions, $\rho^*$ and $v(b)$ are well defined and there exists a stationary deterministic optimal policy $\pi^*$ which maps a belief state to an arm to pull (an action that maximizes the right side of \eqref{equ:Bellman-thm}). Various (approximate) methods have been proposed to solve
	\eqref{equ:Bellman-thm} and to find the optimal policy for the belief MDP, see e.g. \cite{Yu2004, saldi2017asymptotic, sharma2020approximate}.
	In this work, we do not focus on this planning problem for a known model, and we assume the access to an optimization oracle that solve \eqref{equ:Bellman-thm} and returns the optimal average reward $\rho^*$ and the optimal stationary policy for a given known model.





	\section{The SEEU Algorithm}
	This section describes our learning algorithm for the regime switching MAB model: the Spectral Exploration and Exploitation with UCB (SEEU) algorithm.

	To device a learning policy for the POMDP with unknown $\bm \mu $ and $P$,
	one needs a procedure to estimate those quantities from observed rewards.
	\cite{Anandkumar2012,Anandkumar2014} propose the so-called spectral estimator for the unknown parameters in hidden Markov models (HMMs).
	However, the algorithm is not directly applicable to ours, because there is no decision making in HMMs.
	To use the spectral estimator,
	we divide the learning horizon $T$ into nested ``exploration'' and ``exploitation'' phases.
	In the exploration phase, we randomly select an arm in each period.
	This transforms the system into a HMM so that we can apply the spectral method to estimate $\bm\mu$ and $P$ from the observed rewards in the phase.
	In the exploitation phase, based on the estimators obtained from the exploration phase,
	we use a UCB-type policy to further narrow down the optimal belief-based policy in the POMDP introduced in Section~\ref{subsec:belief MDP}.

	\subsection{Spectral Estimator} \label{sec:spectral}
	We introduce the spectral estimator \cite{Anandkumar2012,Anandkumar2014},
	and adapt it to our setting.



	To simplify the notation, suppose the exploration phase starts from period 1 until period $n$, with realized arms $\{i_1,\dots,i_n\}$, and realized rewards $\{r_1,\dots,r_n\}$ sampled from Bernoulli distributions.
	Recall $I$ is the cardinality of the arm set $\I$,
	then one can create a one-to-one mapping from a pair $(i,r)$ into a scalar $s\in\{1,2,...,2I\}$.
	Therefore, the pair can be expressed as a vector $y\in \{0,1\}^{2I}$ such that in each period $t$, $y_t$ satisfies $\1_{\{y_t=e_s\}}=\1_{\{r_t=r,i_t=i\}}$,
	where $e_s$ is a basis vector with its $s$-th element being one and zero otherwise.
	Let $A\in \R^{2I\times M}$ be the observation probability matrix conditional on the state: $A(s,m)=\Prob(R_t=r,I_t=i|M_t=m).$
	It can be shown that $A$ satisfies
	$\E[y_t|M_t=m]=Ae_m$, and $\E[y_{t+1}|M_t=m]=AP^Te_m.$
	For three consecutive observations $y_{t-1},y_{t},y_{t+1}$, define
	\begin{align}\label{def:statistics}
	&\widetilde{y}_{t-1}\coloneqq\E[y_{t+1}\otimes y_t]\E[y_{t-1}\otimes y_t]^{-1}y_{t-1},
	&&\widetilde{y}_{t}\coloneqq\E[y_{t+1}\otimes y_{t-1}]\E[y_{t}\otimes y_{t-1}]^{-1}y_{t},\\
	&M_2\coloneqq\E[\widetilde{y}_{t-1}\otimes\widetilde{y}_{t}],
	&&M_3\coloneqq\E[\widetilde{y}_{t-1}\otimes\widetilde{y}_{t}\otimes y_{t+1}],
	\end{align}
	where $\otimes$ represents the tensor product.


	From the observations $\{y_1,y_2,\dots,y_n\}$, we may construct the estimations $\hat{M}_2$ and $\hat{M}_3$ for $M_2$ and $M_3$ respectively, and apply the tensor decomposition to obtain the estimator $\hat{\bm\mu}$ for the unknown mean reward matrix and $\hat{P}$ for the transition matrix. This procedure is summarized in Algorithm~\ref{alg:example}.
	We use $\hat{A}_m$ (respectively $\hat{B}_m$) to denote the $m$-th column vector of $\hat{A}$ (respectively $\hat{B}$).


	\begin{algorithm}[htb!]
		\caption{The subroutine to estimate $\bm \mu$ and $P$ from the observations from the exploration phase.}
		\label{alg:example}
		\begin{algorithmic}[1] 
			\REQUIRE sample size $n$, $\{y_1,y_2,\dots,y_n\}$ created from the rewards $\{r_1,\dots,r_n\}$ and arms $\{i_1,\dots,i_n\}$
			\ENSURE The estimation $\hat{\bm\mu},\hat{P}$
			\STATE For $i,j\in\{-1,0,1\}$: compute $\hat{W}_{i,j}=\frac{1}{N-2}\sum_{t=2}^{N-1}y_{t+i}\otimes y_{t+j}$.
			\STATE For $t=2,\dots,n-1$: compute $\hat{y}_{t-1}\coloneqq\hat{W}_{1,0}(\hat{W}_{-1,0})^{-1}y_{t-1}$, $\hat{y}_{t}\coloneqq\hat{W}_{1,-1}(\hat{W}_{0,-1})^{-1}y_{t}$.
			\STATE Compute
			$\hat{M}_2\coloneqq\frac{1}{N-2}\sum_{t=2}^{N-1}\hat{y}_{t-1}\otimes \hat{y}_{t}$,
			$\hat{M}_3\coloneqq\frac{1}{N-2}\sum_{t=2}^{N-1}\hat{y}_{t-1}\otimes \hat{y}_{t}\otimes  y_{t+1}$.
			\STATE Apply tensor decomposition (\cite{Anandkumar2014}):\\ $\hat{B}=\textbf{TensorDecomposition}(\hat{M}_2,\hat{M}_3)$.
			\STATE Compute $\hat{A}_m=\hat{W}_{-1,0}(\hat{W}_{1,0})^ \dagger \hat{B}_m$ for each $m\in\mathcal{M}$.
			\STATE Return $m$th row vector $(\hat{\bm\mu})^{m}$ of  $\hat{\bm\mu}$ from $\hat{A}_m$ .
			\STATE Return $\hat{P}= (\hat{A}^ \dagger \hat{B})^{\top}$ ($\dagger$ represents the pseudoinverse of a matrix)
		\end{algorithmic}
	\end{algorithm}
	In addition, we can obtain the following result from \cite{Azizzadenesheli2016} and it provides the confidence regions of the estimators in Algorithm~\ref{alg:example}.

	\begin{proposition}\label{prop:spectral}
		Under Assumptions \ref{assum:Markov Chain} and \ref{assum:reward_matrix2}, for any $\delta \in (0,1)$ and any initial distribution, there exists $N_0$ such that when $n \geq N_0$,
		with probability $1-\delta$, the estimated $ \hat{\bm\mu}$ and $\hat{P}$ by Algorithm~\ref{alg:example} satisfy
		\begin{align}
		||(\bm \mu)^{m}- (\hat{\bm\mu})^{m}||_2 &\leq C_1\sqrt{\frac{\log(6\frac{S^2+S}{\delta})}{n}}, \quad m \in \mathcal{M},\\
		||P - \hat{P}||_2&\leq C_2\sqrt{\frac{\log(6\frac{S^2+S}{\delta})}{n}}. \label{def:confidence bound}
		\end{align}
		where $(\bm \mu)^m$ and $(\hat{\bm \mu})^m$ are the $m$-th row vectors of $\bm\mu$ and $\hat{\bm \mu}$, respectively. Here, $S=2I$, and $C_1$, $C_2$ are constants independent of $n$.
	\end{proposition}
	The expressions of constants $N_0, C_1, C_2$ are given in Section~\ref{sec:proof-prop-spectral} in the appendix. Note that parameters ${\bm\mu},{P}$ are identifiable up to permutations of the hidden states \cite{Azizzadenesheli2016}.

	\subsection{The SEEU Algorithm}
	The SEEU algorithm proceeds in episodes of increasing length, similar to UCRL2 in \cite{Jaksch2010}.
	As mentioned before, each episode is divided into exploration and exploitation phases.
	In episode $k$, it starts with the exploration phase that lasts for a fixed number of periods $\tau_1$, and
	the algorithm uniformly randomly chooses an arm and observes the rewards.
	After the exploration phase, the algorithm applies Algorithm~\ref{alg:example} to (re-)estimate $\bm\mu$ and $P$.
	Moreover, it constructs a confidence interval based on Proposition~\ref{prop:spectral} with a confidence level $1-\delta_k$, where $\delta_k\coloneqq\delta/k^3$ is a vanishing sequence.
	Then the algorithm enters the exploitation phase.
	Its length is proportional to $\sqrt{k}$.
	In the exploitation phase, it conducts UCB-type learning: the arm is pulled according to a policy that corresponds to the optimistic estimator of $\bm \mu$ and $P$ inside the confidence interval.
	The detailed steps are listed in Algorithm \ref{alg:SEEU}.  


	\begin{algorithm}[h!]
		\caption{The SEEU Algorithm}
		\label{alg:SEEU}
		\begin{algorithmic}[1]
			\REQUIRE Initial belief $b_1$, precision $\delta$, exploration parameter $\tau_1$, exploitation parameter $\tau_2$
			\FOR {$k=1,2,3,\dots$}
			\STATE Set the start time of episode $k$, $t_k\coloneqq t$
			\FOR{$t=t_k,t_k+1,\dots,t_k+\tau_1$}
			\STATE Uniformly randomly select an arm: $\Prob(I_t=i)=\frac{1}{I}$
			\ENDFOR
			\STATE\label{step:all-samples-exploration} Input the realized actions and rewards in all previous exploration phases
			$\hat{\mathcal{I}}_k \coloneqq \{i_{t_1:t_1+\tau_1},\cdots,i_{t_k:t_k+\tau_1}
			\}$ and
			$\hat{\mathcal{R}}_k \coloneqq \{r_{t_1:t_1+\tau_1},\cdots,r_{t_k:t_k+\tau_1}
			\}$ to Algorithm~\ref{alg:example} to compute
			$\hat{\bm\mu}_k,\hat{P}_k=\textbf{Spectral Estimation}(\hat{\mathcal{I}}_k, \hat{\mathcal{R}}_k)$
			\STATE Compute the confidence interval $\mathcal{C}_k(\delta_k)$ from \eqref{def:confidence bound} using the confidence level $1-\delta_k=1-\frac{\delta}{k^3}$ such that $\Prob\{(\bm{\mu},P)\in\mathcal{C}_k(\delta_k)\}\geq 1-\delta_k$
			\STATE\label{step:optimistic-pomdp} Find the optimistic POMDP  in the confidence interval\\
			\centerline{$(\bm\mu_k,P_k)=\argmax_{(\bm\mu,P)\in\mathcal{C}(\delta_k)}\rho^*(\bm\mu,P)$}
			\FOR{$t=1,2,\dots,t_k+\tau_1$}
			\STATE Update belief $b_t^k$ to $b_{t+1}^k=H_{\bm{\mu}_k,P_k}(b_t^k,i_t,r_t)$ under the new parameters $(\bm\mu_k,P_k)$
			\ENDFOR
			\FOR{$t=t_k+\tau_1+1,\dots,t_k+\tau_2\sqrt{k}$}
			\STATE Execute the optimal policy $\pi^{(k)}$ by solving the Bellman equation \eqref{equ:Bellman-thm} under parameters $(\bm\mu_k, P_k)$: $i_t=\pi^{(k)}(b_t^k)$
			\STATE Observe reward $r_t$ \label{step:optimal-pomdp-policy}
			\STATE Update the belief at $t+1$ by
			$b_{t+1}^k=H_{\bm\mu_k,P_k}(b_t^k,i_t,r_t)$
			\ENDFOR
			\ENDFOR
		\end{algorithmic}
	\end{algorithm}



	\subsection{Discussions on the SEEU Algorithm}
	\textbf{Computations.}
	For given parameters $(\bm\mu,P)$, we need to compute the optimal average reward $\rho^*(\bm\mu,P)$ that depends on the parameters (Step \ref{step:optimistic-pomdp} in Algorithm \ref{alg:SEEU}). Various computational and approximation methods have been proposed to tackle this planning problem for belief MDPs as mentioned in Section~\ref{subsec:belief MDP}.
	In addition, we need
	to find out the optimistic POMDP in the confidence region $\mathcal{C}_k(\delta_k)$ with the best average reward. For low dimensional models,  one can discretize $\mathcal{C}_k(\delta_k)$ into grids and calculate the corresponding optimal average reward $\rho^*$ at each grid point so as to find (approximately) the optimistic model $(\bm\mu_k, P_k)$.
	However, in general it is not clear whether there is an efficient computational method to find the optimistic plausible POMDP model in the confidence region when the unknown parameters are high-dimensional. This issue is also present in recent studies on learning continuous-state MDPs with the upper confidence bound approach, see e.g. \cite{Lakshmanan2015} for a discussion. In our regret analysis below, we do not take into account approximation errors arising from the computational aspects discussed above, as in \cite{Ortner2012, Lakshmanan2015}.
	To implement the algorithm efficiently (especially for high-dimensional models) remains an intriguing  future direction.

	\textbf{Dependence on the unknown parameters.}
	When computing the confidence region in Step~7 of Algorithm~\ref{alg:SEEU}, the agent needs the information of the constants $C_1$ and $C_2$ in Proposition~\ref{prop:spectral}. These constants depend on a few ``primitives'' that can be hard to know, for example, the mixing rate of the underlying Markov chain.
	However we only need upper bounds for $C_1$ and $C_2$ for the theoretical guarantee, and hence a rough and conservative estimate would be sufficient. Such dependence on some unknown parameters is common in learning problems, and one remedy is to dedicate the beginning of the horizon to estimate the unknown parameters, which typically doesn't increase the rate of the regret.
	Alternatively, $C_1$ and $C_2$ can be replaced by parameters that are tuned by hand.
	See Remark 3 of \cite{Azizzadenesheli2016} for a further discussion on this issue.

	\section{Regret Bound}
	We now give the regret bound for the SEEU algorithm.
	The first result is a high probability bound.


	\begin{theorem}\label{thm:upper_bound}
		Suppose Assumptions 1 to 3 hold. Fix the parameter $\tau_1$ in Algorithm \ref{alg:SEEU} to be sufficiently large. Then there exist constants
		$T_0,C$ which are independent of $T$, such that
		when $T>T_0$, with probability $1-\frac{7}{2}\delta$,
		the regret of Algorithm~\ref{alg:SEEU} satisfies
		\begin{align}
		\mathcal{R}_T& \leq CT^{2/3}\sqrt{\log\left(\frac{3(S+1)}{\delta}T\right)}+T_0\rho^*,
		\end{align}
		where $S=2I$ and $\rho^*$ denotes the optimal long-run average reward under the true model.
	\end{theorem}


	The constant $T_0$ measures the number of periods needed for the sample size in the exploration phases to exceed $N_0$ arising in Proposition~\ref{prop:spectral}.
	The constant $C$ has the following expression:
	\begin{align}
	C&=3\sqrt{2}\left[\left(D+1+\left(1+\frac{D}{2}\right)L_1\right)M^{3/2}C_1+ \left(1+\frac{D}{2}\right)L_2M^{1/2}C_2\right]\tau_2^{1/3}\tau_1^{-1/2}\\
	& \quad + 3\tau_2^{-2/3}(\tau_1\rho^*+D)+(D+1) \sqrt{2 \ln(\frac{1}{\delta})}.
	\end{align}
	Here, $M$ is the number of hidden states, $C_1, C_2$ are given in Proposition~\ref{prop:spectral}, $D$ is a uniform bound on the bias span given in Proposition~\ref{prop:span-uni-bound}, and $L_1, L_2$ are given in Proposition~\ref{prop:lip_bt}. For technical reasons we require $\tau_1$ to be large so as to apply Proposition~\ref{prop:spectral} from the first episode. The impact of $\tau_1$ on the regret will be studied numerically in Section~\ref{sec:numerical}.




	\begin{remark} [Dependancy of $C$ on model parameters]
		The dependency of $C$ on $C_1$ and $C_2$ is directly inherited from the confidence bounds in Proposition~\ref{prop:spectral}, where the constants
		depend on arm similarity, the mixing time and the stationary distribution $\omega$ of the Markov chain (see also \cite{Azizzadenesheli2016}).
		In addition, $C$ depends on $L_1, L_2$ (and hence $\epsilon$, the minimum entry of $P$), which arises from controlling the propagated error in belief in hidden Markov models (Proposition~\ref{prop:lip_bt}; see also \cite{DeCastro2017}).
		Finally, $C$ depends on the bound $D$ of the bias span (see also \cite{Jaksch2010}) in Proposition~\ref{prop:span-uni-bound}. The constant $C$ may not be tight, and its dependence on some parameters may be just an artefact of our proof, but it is the best bound we can obtain.
	\end{remark}

	From Theorem~\ref{thm:upper_bound}, we can choose an appropriate $\delta=\frac{3(S+1)}{T}$ and readily obtain the following expectation bound on the regret. The proof is omitted.
	\begin{theorem}\label{thm:expected_upper_bound}
		Under the same assumptions as Theorem~\ref{thm:upper_bound}, the regret of Algorithm~\ref{alg:SEEU} satisfies
		\begin{align}
		\E[\mathcal{R}_T]& \leq CT^{2/3}\sqrt{2\log T} + (T_0+11(S+1))\rho^*.
		\end{align}
	\end{theorem}

	\begin{remark}[Lower bound]
		For the lower bound of the regret,
		consider the $I$ problem instances (equal to the number of arms):
		In instance $i$, let $\mu_{m,i} = 0.5+ \bar \epsilon$ for all $m$ for a small positive constant $\bar \epsilon$, and let $\mu_{m,j}=0.5$ for all $m$ and $j\neq i$.
		Such structure makes sure that the oracle policy simply pulls one arm without the need to infer the state.
		Since the problem reduces to the classic MAB, the regret is at least $O(\sqrt{IT})$ in this case.
		Note that the setup of the instances may violate Assumption 2, but this can be easily fixed by introducing an arbitrarily small perturbation to $\bm\mu$.
		The gap between the upper and lower bounds is probably caused by the split of exploration/exploitation phases in our algorithm, which resembles the $O(T^{2/3})$ regret for explore-then-commit algorithms in classic multi-armed bandit problems (see Chapter 6 in \citealt{Lattimore2018}).
		We cannot integrate the two phases because of the following fundamental barrier: the spectral estimator cannot use samples generated from the belief-based policy due to the history dependency.
		This is also why \citet{Azizzadenesheli2016} focus on memoryless policies in POMDPs to apply the spectral estimator.
		Since the spectral estimator is the only method that provides finite-sample guarantees for HMMs, we leave the gap as a future direction of research. Nevertheless, we are not aware of other algorithms that can achieve sublinear regret in our setting.
	\end{remark}

	\begin{remark}[General reward distribution]
		Our model requires the estimation of the reward distribution instead of just the mean to estimate the belief. For discrete rewards taking $O$ possible values, the regret bound holds with $S = OI$ instead of $2I$ for Bernoulli rewards. For continuous reward distributions,
		it might be possible to combine the non-parametric HMM inference method in \cite{DeCastro2017} with our algorithm design to obtain regret bounds. However, this requires different analysis techniques and we leave it for future work.
	\end{remark}

	\section{Numerical Experiment} \label{sec:numerical}
	In this section, we present proof-of-concept experiments. Note that large-scale POMDPs with long-run average objectives (the oracle in our problem) are computationally difficult to solve \cite{chatterjee2019complexity}. On the other hand, while there can be many hidden states in general, often only two or three states are important to model in several application areas, e.g. ``bull market'' and ``bear market'' in finance \cite{dai2010trend}.
	Hence we focus on small-scale experiments, following some recent literature on reinforcement learning for POMDPs \citep{Azizzadenesheli2016, Igl2018}.


	As a representative example, we consider a 2-hidden-state, 2-arm setting with 
	$P= \left[
	\begin{matrix}
	1/3 & 2/3 \\
	3/4 & 1/4
	\end{matrix}
	\right] $
	and
	$\bm\mu= \left[
	\begin{matrix}
	0.9 & 0.1 \\
	0.5 & 0.6
	\end{matrix}
	\right] $,
	where the random reward follows a Bernoulli distribution.
	We compare our algorithm with simple heuristics $\epsilon$-greedy ($\epsilon=0.1$), and non-stationary bandits algorithms including Sliding-Window UCB (SW-UCB) \citep{Garivier2011} with tuned window size and Exp3.S \citep{auer2002nonstochastic} with $L=T$, where the hyperparameter $L$ is the number of changes in their algorithm. In Figure \ref{Fig}(a), we plot the average regret versus $T$ of the four algorithms in log-log scale, where the number of runs for each algorithm is 50.
	We observe that the slopes of all algorithms except for SEEU are close to one, suggesting that they incur linear regrets.
	This is expected, because these algorithms don't take into account the hidden states.
	On the other hand, the slope of SEEU is close to $2/3$. This is consistent with our theoretical result (Theorem~\ref{thm:expected_upper_bound}). Similar observations are made on other small-scale examples.
	This demonstrates the effectiveness of our SEEU algorithm, particularly when the horizon length $T$ is relatively large. 

	We also briefly discuss the impact of parameters $\tau_1$ and $\tau_2$ on the algorithm performance.
	For the example above, we calculate the average regret for several pairs of parameters $(\tau_1, \tau_2)$.
	It can be seen that the choices of these parameters do not affect the order $O(T^{2/3})$ of the regret (the slope).
	See Figure~\ref{Fig}(b) for an illustration.

	\begin{figure}
		\centering
		\subfigure{
			\includegraphics[width=7.2cm, height=5.2cm]{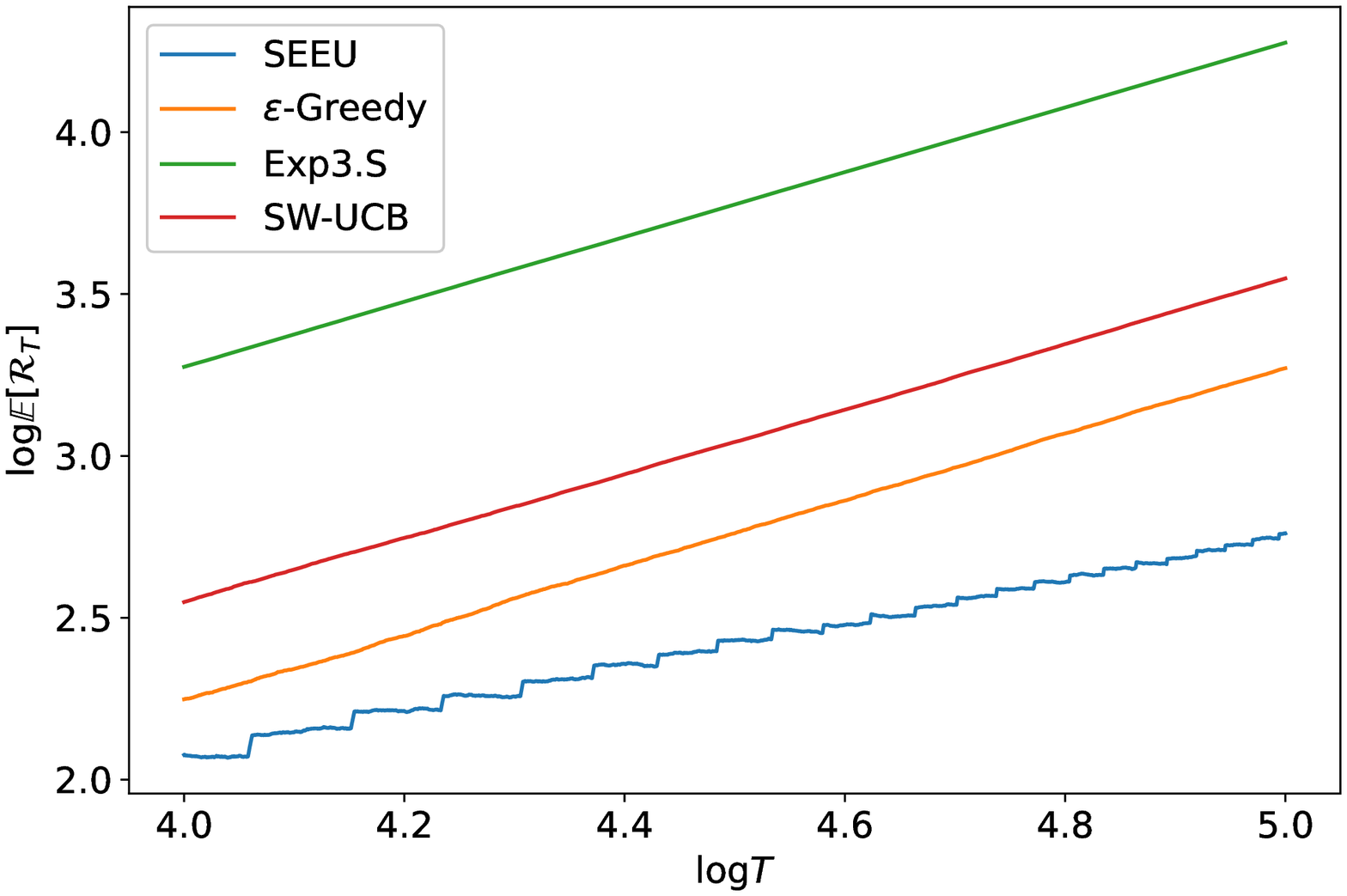} }
		\subfigure{
			\includegraphics[width=7.2cm, height=5.2cm]{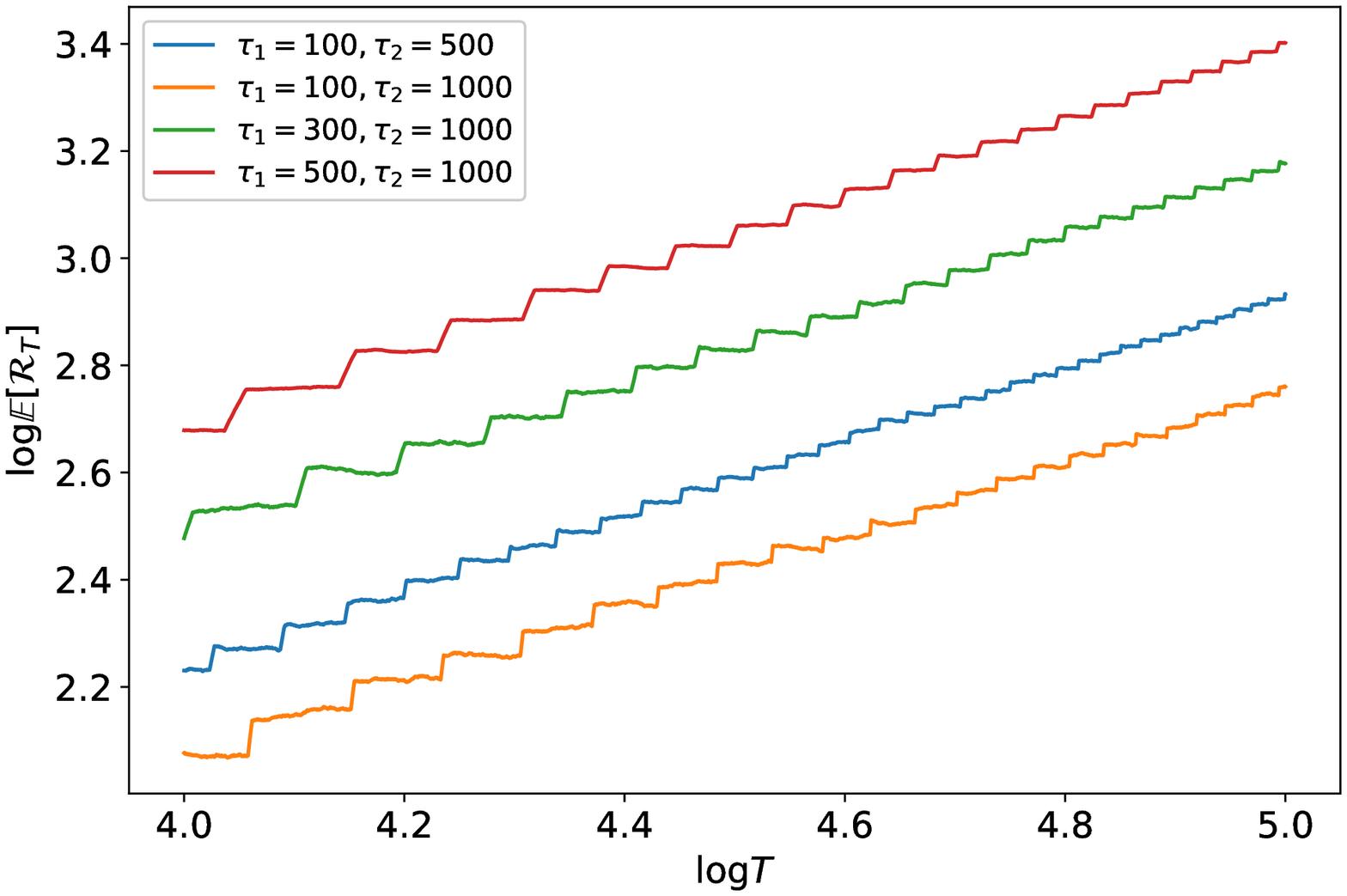}}
		\caption{(a) Regret comparison of four algorithms; (b) Effect of $(\tau_1, \tau_2)$ on the regret.}
		\label{Fig}
	\end{figure}

	\section{Analysis: Proof Sketch for Theorem \ref{thm:upper_bound}} \label{sec:proof1}



	We need the following two technical results. The proof of Proposition~\ref{prop:span-uni-bound} is given in the appendix. The proof of Proposition~\ref{prop:lip_bt} largely follows the proof of Proposition 3 in \cite{DeCastro2017}, with minor changes to take into account the action sequence, so we omit details.



	\begin{proposition}[Uniform bound on the bias span]\label{prop:span-uni-bound}
		If the belief MDP satisfies Assumption~\ref{assum:reward_matrix3}, then
		for $(\rho, v)$ satisfying the Bellman equation \eqref{equ:Bellman-thm},
		we have the span of the bias function $\text{span}(v) \coloneqq\max_{b \in \mathcal{B}}v(b)-\min_{b\in\mathcal{B}}v(b)$ is bounded by $D(\epsilon)$, where
		\begin{align} \label{span-bound-constant}
		D(\epsilon)\coloneqq\frac{8\left( \frac{2}{(1- \alpha)^2}+(1+  \alpha) \log_{\alpha} \frac{1- \alpha}{8}\right)}{1- \alpha}, \quad \text{with $\alpha=\frac{1- 2 \epsilon}{1 - \epsilon} \in (0,1)$.}
		\end{align}
	\end{proposition}

	Recall $v_k$ is the bias function for the optimistic belief MDP in episode $k$. Proposition~\ref{prop:span-uni-bound} guarantees that $span(v_k)$ is bounded by $D=D(\epsilon/2)$ uniformly in $k$,
	because Assumption~\ref{assum:reward_matrix3} can be satisfied (with $\epsilon$ replaced by $\epsilon/2$) by the optimistic MDPs when $T$ is sufficiently large due to Proposition~\ref{prop:spectral}.

	\begin{proposition}[Controlling the belief error]\label{prop:lip_bt}
		Suppose Assumption \ref{assum:reward_matrix3} holds.
		Given $(\hat {\bm{\mu}},\hat P)$, an estimator of the true model parameters $(\bm\mu, P)$.
		For an arbitrary reward-action sequence $\{r_{1:t},i_{1:t}\}_{t\geq 1}$, let $\hat b_t$ and $b_t$ be the corresponding beliefs in period $t$ under $(\hat {\bm{\mu}},\hat P)$ and $(\bm\mu, P)$ respectively.  Then there exists constants $L_1,L_2$ such that
		\begin{align}
		||\hat b_t-b_t||_1\leq L_1|| \hat {\bm{\mu}} -\bm{\mu}||_1+L_2||\hat P- P||_F,
		\end{align}
		where $L_1=4M(\frac{1-\epsilon}{\epsilon})^2/\min\left\{\bm\mu_{\min},1-\bm\mu_{\max}\right\}$, $L_2=4M(1-\epsilon)^2/\epsilon^3+\sqrt{M}$, $||\cdot||_F$ is the Frobenius norm,
		$\bm\mu_{\max}$ and $\bm\mu_{\min}$ are the maximum and minimum element of the matrix $\bm\mu$ respectively.
	\end{proposition}

	\noindent \textbf{Proof Sketch for Theorem \ref{thm:upper_bound}.}
	We provide a proof sketch for Theorem \ref{thm:upper_bound}. The complete proof is given in Section~\ref{sec:proof-complete} in the appendix.
	From the definition of regret in \eqref{def:reg}, we have
	\begin{align}
	\mathcal{R}_T 
	=\sum_{t=1}^{T}(\rho^*-\E^\pi[R_t|\mathcal{F}_{t-1}])+\sum_{t=1}^{T}(\E^\pi[R_t|\mathcal{F}_{t-1}]-R_t),\label{reg_ps}
	\end{align}
	where $\pi$ is our learning algorithm, and $\mathcal{F}_{t-1}$ is the filtration generated by the arms pulled and rewards received under policy $\pi$ up to time $t-1$, and $R_t$ (superscript $\pi$ in \eqref{def:reg} omitted here for notation simplicity) is the reward received under the policy $\pi$.
	For the second term of equation \eqref{reg_ps}, we can use its martingale property and apply the Azuma-Hoeffding inequality to obtain
	\begin{align}\label{reg:expreward-reward_ps}
	\Prob\left(\sum_{t=1}^{T}(\E^\pi[R_t|\mathcal{F}_{t-1}]-R_t) \geq \sqrt{2T\ln{\frac{1}{\delta}}}\right) \leq \delta.
	\end{align}
	For the first term of \eqref{reg_ps}, it can be rewritten as $\sum_{t=1}^{T}(\rho^*-\bar{c}(b_t,I_t)),\label{reg1_ps}$
	where $\bar{c}(b,i)$
	denotes the expected reward of arm $i$ given belief state $b$ under the true model. Denote by $H_k$ and $E_k$ the exploration and exploitation phases at episode $k$ respectively. We
	bound the first term of \eqref{reg_ps} and proceed with three steps.

	{\it Step 1: Bounding the regret in exploration phases.}
	Since the reward is non-negative we have
	\begin{align}\label{reg:explore-bound_ps}
	\sum_{k=1}^{K}\sum_{t\in H_k}(\rho^*-\bar{c}(b_t,I_t)) \leq\sum_{k=1}^{K}\sum_{t\in H_k}\rho^*=K\tau_1\rho^*.
	\end{align}

	{\it Step 2: Bounding the regret in exploitation phases.}
	We define a success event if the true POMDP environment $(\bm \mu, P)$ lies in confidence regions $\mathcal{C}_k(\delta_k)$ for all $k=1, \ldots, K$, and a failure event to be its complement.
	By the choice $\delta_k=\delta/k^3$, one can show that the probability of failure events is at most $\frac{3}{2}\delta$. Next
	note that in the beginning of exploitation phase $k$, SEEU algorithm chooses an optimistic POMDP among the plausible POMDPs, and we denote its corresponding reward matrix, value function, and the optimal average reward by $\bm\mu_k,v_k$ and $\rho^k$, respectively. Thus, on the
	success event we have $\rho^*\leq \rho^k$ for all episode $k$. So we obtain
	\begin{align} \label{eq:regret-exploit}
	&\sum_{k=1}^{K}\sum_{t\in E_k}(\rho^*-\bar{c}(b_t,I_t))
	\leq  \sum_{k=1}^{K}\sum_{t\in E_k}(\rho^k-\bar{c}_k(b_t^k,I_t))+(\bar{c}_k(b_t^k,I_t)-\bar{c}(b_t,I_t)),
	\end{align}
	where $\bar{c}_k$ and $b_t^k$ denotes the expected reward and the belief under the optimistic model in episode $k$.
	Using the Bellman equation~\eqref{equ:Bellman-thm} for $\rho^k$ and Proposition~\eqref{prop:lip_bt} to carefully control the error in the belief transition kernel which is not directly estimated (different from \cite{Jaksch2010}), we can show with high probability, the first term of \eqref{eq:regret-exploit} is bounded by
	$$KD+D \sqrt{2T \ln(\frac{1}{\delta})} +\sum_{k=1}^{K}\sum_{t\in E_k} D\left[ || (\bm{\mu})_{I_t} -(\bm{\mu}_k)_{I_t}||_1+\frac{L_1}{2} || \bm{\mu} -\bm{\mu}_k||_1+\frac{L_2}{2}||P-  P_k||_F\right],$$
	where $(\bm{\mu}_k)_{i}$ is the $i-$th column vector of the matrix $\bm{\mu}_k$, and $D$ is the uniform bound on $span(v_k)$ from Proposition~\ref{prop:span-uni-bound}.
	One can directly bound the second term of \eqref{eq:regret-exploit} by
	$\sum_{k=1}^{K}\sum_{t\in E_k} ||(\bm{\mu}_k)_{I_t}-(\bm{\mu})_{I_t}||_1+||b_t^k-b_t||_1.$
	Applying Proposition \ref{prop:lip_bt} to bound $||b_t^k-b_t||_1$ then using Proposition \ref{prop:spectral},
	one can infer that the regret in exploitation phases is bounded by $O(K \sqrt{\log K})$.


	{\it Step 3: Bounding the number of episodes $K$.}
	A simple calculation of $\sum_{k=1}^{K-1} (\tau_1 + \tau_2 \sqrt{k}) \leq T \leq \sum_{k=1}^K (\tau_1 + \tau_2 \sqrt{k})$ suggests that the number of episodes $K$ is of order $O(T^{2/3})$. Thus, we obtain a regret bound of order $O(T^{2/3}\sqrt{\log T})$.

	\section{Conclusion and Open Questions}
	In this paper, we study a non-stationary MAB model with Markovian regime-switching rewards. We propose a learning algorithm that integrates spectral estimators for hidden Markov models and upper confidence methods from reinforcement learning. We also establish a regret bound of order of $O(T^{2/3}\sqrt{\log T})$ for the learning algorithm. As far as we know, this is the first algorithm with sublinear regret for MAB with unobservable regime switching.

	There are a few important open questions. First,
	it would be interesting to find out whether one can improve the regret bound $O(T^{2/3}\sqrt{\log T})$.
	A related open question is whether the spectral method can be applied to samples generated from adaptive policies, so that the exploration and exploitation can be integrated to improve the theoretical bound.
	Finally, it is not clear whether there is an efficient computational method to find the optimistic plausible POMDP model in the confidence region. We leave them for future research.



\clearpage
\bibliography{ref}

\begin{thebibliography}{42}
\providecommand{\natexlab}[1]{#1}
\providecommand{\url}[1]{\texttt{#1}}
\expandafter\ifx\csname urlstyle\endcsname\relax
  \providecommand{\doi}[1]{doi: #1}\else
  \providecommand{\doi}{doi: \begingroup \urlstyle{rm}\Url}\fi

\bibitem[Agrawal and Jia(2017)]{Agrawal17}
S.~Agrawal and R.~Jia.
\newblock Optimistic posterior sampling for reinforcement learning: worst-case
  regret bounds.
\newblock In \emph{In Advances in Neural Information Processing Systems}, pages
  1184--1194, 2017.

\bibitem[Anandkumar et~al.(2012)Anandkumar, Hsu, and Kakade]{Anandkumar2012}
A.~Anandkumar, D.~Hsu, and S.~M. Kakade.
\newblock A method of moments for mixture models and hidden markov models.
\newblock In \emph{Conference on Learning Theory}, pages 33--1, 2012.

\bibitem[Anandkumar et~al.(2014)Anandkumar, Ge, Hsu, Kakade, and
  Telgarsky]{Anandkumar2014}
A.~Anandkumar, R.~Ge, D.~Hsu, S.~M. Kakade, and M.~Telgarsky.
\newblock Tensor decompositions for learning latent variable models.
\newblock \emph{The Journal of Machine Learning Research}, 15\penalty0
  (1):\penalty0 2773--2832, 2014.

\bibitem[Atan et~al.(2018)Atan, Tekin, and van~der Schaar]{Atan2018}
O.~Atan, C.~Tekin, and M.~van~der Schaar.
\newblock Global bandits.
\newblock \emph{IEEE transactions on neural networks and learning systems},
  29\penalty0 (12):\penalty0 5798--5811, 2018.

\bibitem[Auer and Ortner(2007)]{Auer2007}
P.~Auer and R.~Ortner.
\newblock Logarithmic online regret bounds for undiscounted reinforcement
  learning.
\newblock In \emph{Advances in Neural Information Processing Systems}, pages
  49--56, 2007.

\bibitem[Auer et~al.(2002)Auer, Cesa-Bianchi, Freund, and
  Schapire]{auer2002nonstochastic}
P.~Auer, N.~Cesa-Bianchi, Y.~Freund, and R.~E. Schapire.
\newblock The nonstochastic multiarmed bandit problem.
\newblock \emph{SIAM journal on computing}, 32\penalty0 (1):\penalty0 48--77,
  2002.

\bibitem[Auer et~al.(2019)Auer, Gajane, and Ortner]{auer2019adaptively}
P.~Auer, P.~Gajane, and R.~Ortner.
\newblock Adaptively tracking the best bandit arm with an unknown number of
  distribution changes.
\newblock In \emph{Conference on Learning Theory}, pages 138--158, 2019.

\bibitem[Azizzadenesheli et~al.(2016)Azizzadenesheli, Lazaric, and
  Anandkumar]{Azizzadenesheli2016}
K.~Azizzadenesheli, A.~Lazaric, and A.~Anandkumar.
\newblock Reinforcement learning of {POMDP}s using spectral methods.
\newblock \emph{arXiv preprint arXiv:1602.07764}, 2016.

\bibitem[Azuma(1967)]{Azuma1967}
K.~Azuma.
\newblock Weighted sums of certain dependent random variables.
\newblock \emph{Tohoku Mathematical Journal, Second Series}, 19\penalty0
  (3):\penalty0 357--367, 1967.

\bibitem[Besbes et~al.(2014)Besbes, Gur, and Zeevi]{Besbes2014}
O.~Besbes, Y.~Gur, and A.~Zeevi.
\newblock Stochastic multi-armed-bandit problem with non-stationary rewards.
\newblock In \emph{Advances in neural information processing systems}, pages
  199--207, 2014.

\bibitem[Besbes et~al.(2019)Besbes, Gur, and Zeevi]{Besbes2019}
O.~Besbes, Y.~Gur, and A.~Zeevi.
\newblock Optimal exploration-exploitation in a multi-armed-bandit problem with
  non-stationary rewards.
\newblock \emph{Stochastic Systems}, 2019.

\bibitem[Bouneffouf and Rish(2019)]{Bouneffouf2019}
D.~Bouneffouf and I.~Rish.
\newblock A survey on practical applications of multi-armed and contextual
  bandits.
\newblock \emph{arXiv preprint arXiv:1904.10040}, 2019.

\bibitem[Bubeck and Cesa-Bianchi(2012)]{Bubeck2012}
S.~Bubeck and N.~Cesa-Bianchi.
\newblock Regret analysis of stochastic and nonstochastic multi-armed bandit
  problems.
\newblock \emph{Foundations and Trends{\textregistered} in Machine Learning},
  5\penalty0 (1):\penalty0 1--122, 2012.

\bibitem[Chatterjee et~al.(2019)Chatterjee, Saona, and
  Ziliotto]{chatterjee2019complexity}
K.~Chatterjee, R.~Saona, and B.~Ziliotto.
\newblock The complexity of pomdps with long-run average objectives.
\newblock \emph{arXiv preprint arXiv:1904.13360}, 2019.

\bibitem[Chen et~al.(2020)Chen, Wang, and Wang]{chen2020learning}
N.~Chen, C.~Wang, and L.~Wang.
\newblock Learning and optimization with seasonal patterns.
\newblock \emph{Working paper}, 2020.

\bibitem[Cheung et~al.(2018)Cheung, Simchi-Levi, and Zhu]{Cheung2018}
W.~C. Cheung, D.~Simchi-Levi, and R.~Zhu.
\newblock Learning to optimize under non-stationarity.
\newblock \emph{arXiv preprint arXiv:1810.03024}, 2018.

\bibitem[Dai et~al.(2010)Dai, Zhang, and Zhu]{dai2010trend}
M.~Dai, Q.~Zhang, and Q.~J. Zhu.
\newblock Trend following trading under a regime switching model.
\newblock \emph{SIAM Journal on Financial Mathematics}, 1\penalty0
  (1):\penalty0 780--810, 2010.

\bibitem[De~Castro et~al.(2017)De~Castro, Gassiat, and Le~Corff]{DeCastro2017}
Y.~De~Castro, E.~Gassiat, and S.~Le~Corff.
\newblock Consistent estimation of the filtering and marginal smoothing
  distributions in nonparametric hidden markov models.
\newblock \emph{IEEE Transactions on Information Theory}, 63\penalty0
  (8):\penalty0 4758--4777, 2017.

\bibitem[Fiez et~al.(2019)Fiez, Sekar, and Ratliff]{Fiez2019}
T.~Fiez, S.~Sekar, and L.~J. Ratliff.
\newblock Multi-armed bandits for correlated markovian environments with
  smoothed reward feedback.
\newblock \emph{arXiv preprint arXiv:1803.04008}, 2019.

\bibitem[Garivier and Moulines(2011)]{Garivier2011}
A.~Garivier and E.~Moulines.
\newblock On upper-confidence bound policies for switching bandit problems.
\newblock In \emph{International Conference on Algorithmic Learning Theory},
  pages 174--188, Berlin, Heidelberg, 2011. Springer.

\bibitem[Guha et~al.(2010)Guha, Munagala, and Shi]{Guha2010}
S.~Guha, K.~Munagala, and P.~Shi.
\newblock Approximation algorithms for restless bandit problems.
\newblock \emph{Journal of the ACM (JACM)}, 58\penalty0 (1):\penalty0 3, 2010.

\bibitem[Gupta et~al.(2018)Gupta, Joshi, and Ya{\u{g}}an]{Gupta2018}
S.~Gupta, G.~Joshi, and O.~Ya{\u{g}}an.
\newblock Correlated multi-armed bandits with a latent random source.
\newblock \emph{arXiv preprint arXiv:1808.05904}, 2018.

\bibitem[Hinderer(2005)]{Hinderer2005}
K.~Hinderer.
\newblock Lipschitz continuity of value functions in markovian decision
  processes.
\newblock \emph{Mathematical Methods of Operations Research}, 62\penalty0
  (1):\penalty0 3--22, 2005.

\bibitem[Igl et~al.(2018)Igl, Zintgraf, Le, Wood, and Whiteson]{Igl2018}
M.~Igl, L.~Zintgraf, T.~A. Le, F.~Wood, and S.~Whiteson.
\newblock Deep variational reinforcement learning for pomdps.
\newblock \emph{arXiv preprint arXiv:1806.02426}, 2018.

\bibitem[Jaksch et~al.(2010)Jaksch, Ortner, and Auer]{Jaksch2010}
T.~Jaksch, R.~Ortner, and P.~Auer.
\newblock Near-optimal regret bounds for reinforcement learning.
\newblock \emph{Journal of Machine Learning Research}, 11\penalty0
  (Apr):\penalty0 1563--1600, 2010.

\bibitem[Krishnamurthy(2016)]{Krish2016}
V.~Krishnamurthy.
\newblock \emph{Partially observed Markov decision processes}.
\newblock Cambridge University Press, 2016.

\bibitem[Lakshmanan et~al.(2015)Lakshmanan, Ortner, and Ryabko]{Lakshmanan2015}
K.~Lakshmanan, R.~Ortner, and D.~Ryabko.
\newblock Improved regret bounds for undiscounted continuous reinforcement
  learning.
\newblock In \emph{Proceedings of the International Conference on Machine
  Learning}, pages 524--532, 2015.

\bibitem[Lattimore and Munos(2014)]{Lattimore2014}
T.~Lattimore and R.~Munos.
\newblock Bounded regret for finite-armed structured bandits.
\newblock In \emph{Advances in Neural Information Processing Systems}, pages
  550--558, 2014.

\bibitem[Lattimore and Szepesv{\'a}ri(2018)]{Lattimore2018}
T.~Lattimore and C.~Szepesv{\'a}ri.
\newblock \emph{Bandit algorithms}.
\newblock preprint, 2018.

\bibitem[Lehéricy(2019)]{Lehericy2019}
L.~Lehéricy.
\newblock Consistent order estimation for nonparametric hidden markov models.
\newblock \emph{Bernoulli}, 25\penalty0 (1):\penalty0 464--498, 2019.

\bibitem[Maillard and Mannor(2014)]{Maillard2014}
O.~A. Maillard and S.~Mannor.
\newblock Latent bandits.
\newblock In \emph{International Conference on Machine Learning}, pages
  136--144, 2014.

\bibitem[Mamon and Elliott(2007)]{mamon2007hidden}
R.~S. Mamon and R.~J. Elliott.
\newblock \emph{Hidden Markov models in finance}, volume~4.
\newblock Springer, 2007.

\bibitem[Mersereau et~al.(2009)Mersereau, Rusmevichientong, and
  Tsitsiklis]{Mersereau2009}
A.~J. Mersereau, P.~Rusmevichientong, and J.~N. Tsitsiklis.
\newblock A structured multiarmed bandit problem and the greedy policy.
\newblock \emph{IEEE Transactions on Automatic Control}, 54\penalty0
  (12):\penalty0 2787--2802, 2009.

\bibitem[Ortner and Ryabko(2012)]{Ortner2012}
R.~Ortner and D.~Ryabko.
\newblock Online regret bounds for undiscounted continuous reinforcement
  learning.
\newblock In \emph{Advances in Neural Information Processing Systems}, pages
  1763--1771, 2012.

\bibitem[Ortner et~al.(2014)Ortner, Ryabko, Auer, and Munos]{Ortner2014}
R.~Ortner, D.~Ryabko, P.~Auer, and R.~Munos.
\newblock Regret bounds for restless markov bandits.
\newblock \emph{Theoretical Computer Science}, 558:\penalty0 62--76, 2014.

\bibitem[Qian et~al.(2018)Qian, Fruit, Pirotta, and Lazaric]{Qian2018}
J.~Qian, R.~Fruit, M.~Pirotta, and A.~Lazaric.
\newblock Exploration bonus for regret minimization in undiscounted discrete
  and continuous markov decision processes.
\newblock \emph{arXiv preprint arXiv:1812.04363}, 2018.

\bibitem[Ross(1968)]{ross1968arbitrary}
S.~M. Ross.
\newblock Arbitrary state markovian decision processes.
\newblock \emph{The Annals of Mathematical Statistics}, 39\penalty0
  (6):\penalty0 2118--2122, 1968.

\bibitem[Saldi et~al.(2017)Saldi, Y{\"u}ksel, and Linder]{saldi2017asymptotic}
N.~Saldi, S.~Y{\"u}ksel, and T.~Linder.
\newblock On the asymptotic optimality of finite approximations to markov
  decision processes with borel spaces.
\newblock \emph{Mathematics of Operations Research}, 42\penalty0 (4):\penalty0
  945--978, 2017.

\bibitem[Sharma et~al.(2020)Sharma, Jafarnia-Jahromi, and
  Jain]{sharma2020approximate}
H.~Sharma, M.~Jafarnia-Jahromi, and R.~Jain.
\newblock Approximate relative value learning for average-reward continuous
  state mdps.
\newblock In \emph{Uncertainty in Artificial Intelligence}, pages 956--964.
  PMLR, 2020.

\bibitem[Slivkins(2019)]{Slivkins2019}
A.~Slivkins.
\newblock Introduction to multi-armed bandits.
\newblock \emph{arXiv preprint arXiv:1904.07272}, 2019.

\bibitem[Slivkins and Upfal(2008)]{Slivkins2008}
A.~Slivkins and E.~Upfal.
\newblock Adapting to a changing environment: the brownian restless bandits.
\newblock In \emph{Conference on Learning Theory}, pages 343--354, 2008.

\bibitem[Yu and Bertsekas(2004)]{Yu2004}
H.~Yu and D.~P. Bertsekas.
\newblock Discretized approximations for pomdp with average cost.
\newblock In \emph{Proceedings of the 20th conference on Uncertainty in
  artificial intelligence}, pages 619--627. AUAI Press, 2004.

\end{thebibliography}
\bibliographystyle{plainnat}

\clearpage
\section*{Online Appendix}
	\subsection*{Table of Notations}
	\begin{table}[h!]
		\centering
		\begin{tabular}{c|l}
			\hline
			Notation &Description\\\hline
			$T$ & The length of decision horizon\\
			$P$ & The transition matrix of underlying states\\
			$\bm{\mu}$ & The mean reward matrix\\
			$\mathcal{M}$ & The set of underlying states\\
			$\mathcal{I}$ & The set of arms\\
			$\mathcal{R}$ & The set of rewards\\
			$M_t$ & The underlying state at time $t$\\
			$I_t$ & The chosen arm at time $t$\\
			$R_t$ & The random reward at time $t$\\
			$\mathcal{F}_t$ & The history up to time $t$\\
			$\rho^*$ & The optimal long term average reward\\
			$\mathcal{R}_T$ & Regret during the total horizon\\
			$\bm{b_t}$ & The belief state at time $t$\\
			$\mathcal{B}$ & The set of belief states\\
			$c(m, i)$ & The reward function given the state and arm\\
			$\bar{c}(b,i)$ & The reward function w.r.t. belief state\\
			$Q$ & The reward distribution\\
			$H$ & The belief state forward kernel\\
			$\bar{T}$ & The transition kernel w.r.t. belief state\\
			$D$ & The uniform upper bound of $\text{span}(v)$\\
			$\E^{\pi}$ & Taken expectation respect to the true parameters $\bm{\mu}$ and $P$ under policy $\pi$\\
			$\E_k^{\pi}$ & Taken expectation respect to estimated parameters $\bm{\mu}_k$ and $P_k$ under policy $\pi$\\
			\hline
		\end{tabular}
		\caption{Summary of notations}\label{tab:notations}
	\end{table}

	\section{Constants in Proposition \ref{prop:spectral}} \label{sec:proof-prop-spectral}
	We consider the action-reward pair $(R_t,I_t)$ as our observation of the underlying state $M_t$.
	We encode the pair $(r,i)$ into a variable $s\in\{1,2,...,2 I\}$ through a suitable one-to-one mapping. We rewrite our observable random vector $(R_t,I_t)$ as a random variable $S_t$. Hence we can define the following matrix $A_1,A_2,A_3\in\R^{2I\times M}$, where
	\begin{eqnarray*}
		\begin{array}{ll}
			A_1(s,m)&=\Prob(S_{t-1}=s|M_t=m),\\
			A_2(s,m)&=\Prob(S_{t}=s|M_t=m),\\
			A_3(s,m)&=\Prob(S_{t+1}=s|M_t=m),
		\end{array}
	\end{eqnarray*}
	for $s\in \{1,2,...,2 I\},k\in \mathcal{M}=\{1,2,...,M\}$.
	It follows from Lemma 5, Lemma 8 and Theorem 3 in \cite{Azizzadenesheli2016} that the spectral estimators $\hat{\bm\mu},\hat{P}$ have the following guarantee: pick any $\delta \in (0,1)$, when the number of samples $n$ satisfies
	\begin{equation*}
	n \geq N_0 := \left(\frac{G\frac{2\sqrt{2}+1}{1-\theta}}{\omega_{\min}\sigma^2}\right)^2 \log \left(\frac{2(S^2+S)}{\delta}\right)\max\left\{\frac{16 \times M^{1/3}}{C_0^{2/3}\omega_{\min}^{1/3}},\frac{2\sqrt{2}M}{C_0^2\omega_{\min}\sigma^2},1\right\},
	\end{equation*}
	then with probability $1-\delta$ we have
	\begin{align*}
	||(\bm \mu)^m-(\hat{\bm \mu})^{m}||_2 &\leq C_1\sqrt{\frac{\log(6\frac{S^2+S}{\delta})}{n}}, \\
	||P- \hat{P}||_2&\leq C_2\sqrt{\frac{\log(6\frac{S^2+S}{\delta})}{n}},
	\end{align*}
	for $\: \:m \in \mathcal{M}$ up  to permutation, with
	\begin{align}\label{eq:defC}
	C_1&=\frac{21}{\sigma_{1,-1}}I C_{3},\\
	C_2&=\frac{4}{\sigma_{\min}(A_2)}\left(\sqrt{M} +M\frac{21}{\sigma_{1,-1}}\right)C_{3},\\
	C_{3}&=2G\frac{2\sqrt{2}+1}{(1-\theta)\omega_{\min}^{0.5}}\left(1+\frac{8\sqrt{2}}{\omega_{\min}^{2}\sigma^3}+\frac{256}{\omega_{\min}^{2}\sigma^2}\right),
	\end{align}
	where $S=2I$, $C_0$ is a numerical constant (see Theorem 16 in \cite{Azizzadenesheli2016}), $\sigma_{1,-1}$ is the smallest nonzero singular value of the covariance matrix $\E[y_{t+1}\otimes y_{t-1}]$, and $\sigma=\min\{\sigma_{\min}(A_1),\sigma_{\min}(A_2),\sigma_{\min}(A_3)\}$,  where $\sigma_{\min}(A_i)$ represents the smallest nonzero singular value of the matrix $A_i$, for $i=1,2,3$. In addition, $\omega=(\omega(m))$ represents the stationary distribution of the underlying Markov chain $\{M_t\}$, and $\omega_{\min}=\min_m \omega(m) \ge \epsilon = \min_{i,j}P_{ij}$. Finally, $G$ and $\theta$ are the mixing rate parameters such that
	$$\sup_{m_1}||f_{1\rightarrow t}(\cdot|m_1)-\omega||_{TV}\leq G\theta^{t-1},$$
	where $f_{1\rightarrow t}(\cdot|m_1)$ denotes the probability distribution vector of the underlying state at time $t$, starting from the initial state $m_1$. Under Assumption~\ref{assum:reward_matrix3}, one can take $G=2$ and have the (crude) bound $\theta \le 1 -\epsilon$, see e.g. Theorems 2.7.2 and 2.7.4 in \cite{Krish2016}.

	\section{Proof of Proposition \ref{prop:span-uni-bound}}

	\begin{proof} 
		We first introduce a few notations. Let $V_{\beta}(b)$ be the (optimal) value function of the infinite-horizon discounted version of the POMDP (or belief MDP) with discount factor $\beta \in (0,1)$ and initial belief state $b$.
		Define $v_{\beta}(b)\coloneqq V_{\beta}(b)-V_{\beta}(s)$ for a fixed belief state $s$, where $v_{\beta}$ is the bias function for the \textit{discounted} problem.
		We also introduce $\ell_1$ distance to the belief space $\mathcal{B}$: $\rho_b(b,b'):=\|b-b'\|_1$. For any function $f:\mathcal{B} \mapsto \R$, define the Lipschitz module of a function $f$ by
		\begin{align}\label{def:lip-module}
		l_{\rho_{\mathcal{B}}}(f)\coloneqq\sup\limits_{b\neq b'} \frac{|f(b)-f(b')|}{\rho_b(b,b')}.
		\end{align}

		The main idea of the proof is as follows. To bound $\text{span}(v) :=\max_{b \in \mathcal{B}}v(b)-\min_{b\in\mathcal{B}}v(b)$ where $v$ is the bias function for our \textit{undiscounted} problem, it suffices to bound the Lipschitz module of $v$ due to the fact that $\sup_{b\neq b'}||b-b'||_1= 2$.
		Under our assumptions, it can be shown that the bias function $v$ for the undiscounted problem satisfy the relation
		\begin{equation}\label{eq:limit-disc}
		v(b)=\lim_{\beta \to 1-} v_{\beta}(b), \quad \text{for $b\in\mathcal{B}$.}
		\end{equation}
		Then applying Lemma 3.2(a) \citep{Hinderer2005} yields
		\begin{align} \label{eq:lip-v}
		l_{\rho_{\mathcal{B}}}(v)\leq \lim_{\beta \to 1-}l_{\rho_{\mathcal{B}}}(v_{\beta})=\lim_{\beta \to 1-}l_{\rho_{\mathcal{B}}}(V_{\beta}).  
		\end{align}
		So it suffices to bound $\lim_{\beta \to 1-}l_{\rho_{\mathcal{B}}}(V_{\beta})$. The bound for $l_{\rho_{\mathcal{B}}}(V_{\beta})$ in turn implies that $(v_{\beta})_{\beta}$ is a uniformly bounded equicontinuous family of functions, and hence
		validates \eqref{eq:limit-disc} by Theorem 2 in \cite{ross1968arbitrary}.

		We next proceed to bound $l_{\rho_{\mathcal{B}}}(V_{\beta})$ and we will show that for any $\beta \in (0,1)$ we have  $l_{\rho_{\mathcal{B}}}(V_{\beta}) \le \frac{\eta}{1- \gamma}$ for some constants $\eta>0, \gamma \in (0,1)$ that are both independent of $\beta$. To this end, we consider the finite horizon discounted belief MDP, and
		let $V_{n,\beta}$ be the optimal value function for the discounted problem with horizon length $n$ and discount factor $\beta$.
		Since the reward is bounded, it is readily seen that $\lim_{n \rightarrow\infty}V_{n,\beta}= V_{\beta}$. Then Lemma 2.1(e) \cite{Hinderer2005} suggests that
		\begin{align}\label{lip-module-average}
		l_{\rho_b}(V_{\beta})\leq \liminf\limits_{n\to\infty}l_{\rho_b}(V_{n,\beta}).
		\end{align}
		Thus, to bound $l_{\rho_b}(V_{\beta})$, it suffices to bound the Lipschitz module $l_{\rho_b}(V_{n,\beta})$. The strategy is to apply the results including Lemmas 3.2 and 3.4 in \cite{Hinderer2005}, but it requires a new analysis to verify the conditions there.

		To proceed,
		standard dynamic programming theory states that $V_{n,\beta}(b) = J_{1}(b)$, and $J_{1}(b)$ can be computed by the backward recursion:
		\begin{align}
		J_n(b_n) &= \bar{c}(b_n, i_n), \label{equ:Bellman-thm-finite-discount1} \\
		J_{t}(b_t) &= \max_{i_t \in \mathcal{I}} \left\{\bar{c}(b_t, i_t) + \beta \int_{\mathcal{B}} J_{t+1}(b_{t+1}) \bar{T}(db_{t+1}| b_t, i_t)\right\}, \quad 1 \leq t<n, \label{equ:Bellman-thm-finite-discount2}
		\end{align}
		where $\bar{T}$ is the (action-dependent) one-step transition law of the belief state, and $J_{t}(b_t)$ are finite for each $t$.
		More generally, for a given sequence of actions $i_{1:n},$ the $n$-step transition kernel for the belief state is define by
		\begin{align}\label{def:n-trans-ker}
		\bar{T}^{(n)}(\textbf{A}|b,i_{1:n}) \coloneqq \Prob(b_{n}\in \textbf{A} |b_1=b,i_{1:n}), \quad \textbf{A} \subset \mathcal{B}.
		\end{align}
		To use the results in \cite{Hinderer2005}, we need to study the Lipschitz property of this multi-step transition kernel as we will see later. Following \cite{Hinderer2005},
		we introduce the Lipschitz module for a transition kernel $\phi(b, db')$ on belief states.
		Let $K_{\rho_{\mathcal{B}}}(\nu,\theta)$ be the Kantorovich metric of two probability measure $\nu,\theta$ defined on $\mathcal{B}$:
		\begin{align}\label{def:K-metric}
		K_{\rho_{\mathcal{B}}}(\nu,\theta)\coloneqq\sup_f\left\{\left|\int_{\mathcal{B}} f(b)\nu(db)-\int_{\mathcal{B}} f(b)\theta(db)\right|,f\in\text{Lip}_1(\rho_{\mathcal{B}})\right\},
		\end{align}
		where $\text{Lip}_1(\rho_{\mathcal{B}})$ is the set of functions on $\mathcal{B}$ with Lipschitz module $l_{\rho_{\mathcal{B}}}(f)\leq1$.
		Then the Lipschitz module of the transition kernel $l_{\rho_{\mathcal{B}}}(\phi)$ is defined as:
		\begin{align}\label{def:lip-module-transi-kernel}
		l_{\rho_{\mathcal{B}}}(\phi)\coloneqq\sup_{b^1\neq b^2}\frac{K_{\rho_{\mathcal{B}}}(\phi(b^1, db'),\phi(b^2,db'))}{\rho_{\mathcal{B}}(b^1,b^2)}.
		\end{align}
		The transition kernel $\phi$ is called Lipschitz continuous if $l_{\rho_{\mathcal{B}}}(\phi) < \infty$. To bound $l_{\rho_{\mathcal{B}}}(V_{n,\beta})$ and to apply the results in \cite{Hinderer2005}, the key technical result we need is the following lemma. We defer its proof to the end of this section. Recall that $\epsilon= \min\limits_{i,j \in \mathcal{M}}P_{i,j}>0$.
		\begin{lemma}\label{lemma:lip_module_bound}
			For $1 \leq n < \infty$, the $n$-step belief state transition kernel $\bar{T}^{(n)}(\cdot|b,i_{1:n})$ in \eqref{def:n-trans-ker} is uniformly Lipschitz in $i_{1:n}$, and the Lipschitz module is bounded as follows:
			\begin{align}
			l_{\rho_{\mathcal{B}}}(\bar{T}^{(n)}) \le  C_4\alpha^{n}+C_5,
			\end{align}
			where $C_4=\frac{2}{1 - \alpha}$ and $C_5= \frac{1}{2} + \frac{\alpha}{2}$ with $\alpha=1-\frac{\epsilon}{1-\epsilon} \in (0, 1)$. As a consequence,
			there exist constants $n_0 \in \mathbb{Z}^+$ and $\gamma<1$ such that
			$l_{\rho_{\mathcal{B}}}(\bar{T}^{(n_0)})<\gamma$ for any $i_{1:n}$. Here, we can take $n_0=\lceil \log_{\alpha}\frac{1-C_5}{2C_4} \rceil$, and $\gamma=\frac{1}{2}(1+C_5) = \frac{3 + \alpha}{4}$.
		\end{lemma}

		With Lemma~\ref{lemma:lip_module_bound}, we are now ready to bound $l_{\rho_{\mathcal{B}}}(V_{n,\beta})$. Consider $n=k n_0$ for some positive integer $k$. We can infer from the value iteration in \eqref{equ:Bellman-thm-finite-discount2} that
		\begin{align}
		J_{t}(b_{t})=\sup_{i_{t:t+n_0-1}}&\Big\{\sum_{l=0}^{n_0-1}\beta^l\int_{\mathcal{B}}\bar{c}(b_{t+l},i_{t+l-1})\bar{T}^{(l)}(db_{t+l}|b_{t},i_{t:t+l-1})\\
		&+\beta^{n_0}\int_{\mathcal{B}}J_{t+n_0}(b_{t+n_0})\bar{T}^{(n_0)}(db_{t+n_0}|b_{t},i_{t:t+n_0-1})\Big\}, \quad 1 \leq t \leq n-n_0. \label{equ:n-Bellman-optimality}
		\end{align}
		Bounding $\bar c$ in \eqref{equ:n-Bellman-optimality} by $r_{\max}=1$ (the bound for rewards) and $\bar T^{(l)} $ by its Lipschitz module,
		we obtain the following inequality using Lemmas 3.2 and 3.4 in \cite{Hinderer2005}:
		\begin{align}
		l_{\rho_{\mathcal{B}}}(J_{t})&\leq r_{\max}\cdot \sum_{l=0}^{n_0-1}\beta^t l_{\rho_{\mathcal{B}}}^{{\mathcal{I}}^{l}}(\bar{T}^{(l)})+\beta^{n_0} \cdot l_{\rho_{\mathcal{B}}}^{{\mathcal{I}}^{n_0}}(\bar{T}^{(n_0)}) \cdot l_{\rho_{\mathcal{B}}}(J_{t+n_0}),\label{lip-module-Bellman}
		\end{align}
		where $l_{\rho_{\mathcal{B}}}^{{\mathcal{I}}^{l}}(\bar{T}^{(l)})$ is the supremum of the Lipschitz module $l_{\rho_{\mathcal{B}}}(\bar{T}^{(l)})$ over actions:
		\begin{align}
		l_{\rho_{\mathcal{B}}}^{{\mathcal{I}}^{l}}(\bar{T}^{(l)}) \coloneqq \sup_{i_{t:t+l-1}}\sup_{b_t\neq b'_t}\frac{K_{\rho_{\mathcal{B}}}(\bar{T}^{(l)}(db_{t+l}|b_t,i_{t:t+l-1}),\bar{T}^{(l)}(db_{t+l}|b'_t,i_{t:t+l-1}))}{\rho_{\mathcal{B}}(b_t,b'_t)}, \quad 0 \leq l \leq n_0.
		\end{align}
		Note that the value function in the last period $J_n(b_n)=\bar{c}(b_n, i_n)$ is uniformly Lipschitz in $i_n$ with Lipschitz module $r_{\max}=1$.
		Applying the last inequality iteratively
		for $n_i=1+ i n_0$ with $0 \le i < k$ and by Lemma~\ref{lemma:lip_module_bound},
		we have
		\begin{align}
		l_{\rho_{\mathcal{B}}}(J_{n_i})
		&\leq  \sum_{t=0}^{n_0-1}\beta^t l_{\rho_{\mathcal{B}}}^{{\mathcal{I}}^{t}}(\bar{T}^{(t)})+\beta^{n_0} \cdot \gamma  \cdot l_{\rho_{\mathcal{B}}}(J_{n_{i+1}})\\
		&\leq  \sum_{t=0}^{n_0-1}[C_4 \alpha^t + C_5]+\beta^{n_0} \cdot \gamma  \cdot l_{\rho_{\mathcal{B}}}(J_{n_{i+1}})\\
		&\leq \eta+\beta^{n_0}\gamma \cdot l_{\rho_{\mathcal{B}}}(J_{n_{i+1}}),
		\end{align}
		where
		\begin{equation} \label{eq:eta}
		\eta =  \frac{C_4}{1-\alpha} + C_5 n _0,
		\end{equation}
		and $C_4, C_5, n_0, \alpha$ are given in Lemma~\ref{lemma:lip_module_bound}.
		Iterating over $i$ and using $l_{\rho_{\mathcal{B}}}(J_{n}) = l_{\rho_{\mathcal{B}}}(J_{k n_0}) = r_{\max}$, we obtain
		\begin{align}
		l_{\rho_{\mathcal{B}}}(J_{0})\leq \eta \cdot \frac{1-\left(\beta^{n_0}\gamma \right)^{k}}{1-\beta^{n_0}\gamma} + \left(\beta^{n_0}\gamma \right)^{k} \cdot r_{\max}.
		\end{align}
		Recall that for $n=k n_0,$ $V_{n,\beta}(b) = V_{k n_0,\beta}(b) = J_{0}(b)$. Since $\beta <1$ and $\gamma<1$, we then get
		\begin{align}\label{lip-module-infinite}
		\liminf\limits_{k\rightarrow\infty}l_{\rho_{\mathcal{B}}}(V_{k n_0,\beta}) \leq \frac{\eta}{1-\gamma}.
		\end{align}
		Together with \eqref{eq:lip-v} and \eqref{lip-module-average},
		we can deduce that for a belief MDP satisfying $\min_{i,j \in \mathcal{M}} P_{ij} =\epsilon>0$, the span of the bias function is upper bounded by $$ span(v) \le D(\epsilon):= \frac{2\eta(\epsilon)}{1-\gamma(\epsilon)},$$ where with slight abuse of notations we use $\eta(\epsilon)$ (see \eqref{eq:eta}) and $\gamma(\epsilon)$ (see Lemma~\ref{lemma:lip_module_bound}) to emphasize their dependency on $\epsilon.$ The proof is completed by simplifying the expression of $D(\epsilon)$.
	\end{proof}

	\subsection{Proof of Lemma~\ref{lemma:lip_module_bound}}
	\begin{proof}
		Rewriting the Kantorovich metric, we have:
		\begin{align}
		&K\{\bar{T}^{(n)}(db'|b^1,i_{1:n}),\bar{T}^{(n)}(db'|b^2,i_{1:n}\}\\
		&=\sup_f\left\{\left|\int f(b')\bar{T}^{(n)}(db'|b^1,i_{1:n})-\int f(b')\bar{T}^{(n)}(db'|b^2,i_{1:n})\right|,f\in \text{Lip}_1\right\}\\
		&=\sup_f\left\{\left|\int f(b')\bar{T}^{(n)}(db'|b^1,i_{1:n})-\int f(b')\bar{T}^{(n)}(db'|b^2,i_{1:n})\right|,f\in \text{Lip}_1,||f||_\infty\leq 1\right\}.
		\end{align}
		The last equality follows from the following argument. Note that the span of a function $f$ with Lipschitz module 1 is bounded by $\text{Diam}(\B)$ where $\text{Diam}(\B):=\sup_{b^1\neq b^2}||b^1-b^2||_1= 2$.
		So for any $f\in \text{Lip}_1$ we can find a constant c that $||f+c||_\infty\leq \text{Diam}(\B)/2$. Moreover, let $\phi(f)=|\int f(b')\bar{T}^{(n)}(db'|b^1,i_{1:n})-\int f(b')\bar{T}^{(n)}(db'|b^2,i_{1:n})|$, we know $\phi(f)=\phi(f+c)$ for any constant $c$. Without loss of generality, we can constrain $||f||_\infty \leq \text{Diam}(\B)/2 \le 1$.

		We introduce a few notations to facilitate the presentation. We define the $n$-step reward kernel $\bar{Q}^{(n)}$, where $\bar{Q}^{(n)}(\prod_{t=1}^{n}dr_t|b,i_{1:n})$ is a probability measure on $\mathcal{R}^n$:
		\begin{align}\label{def:n-belief-reward-kernel}
		\bar{Q}^{(n)}(A_1\times...\times A_{n}|b,i_{1:n})=\Prob((r_1,\dots,r_{n})\in A_1\times...\times A_{n}|b,i_{1:n}).
		\end{align}
		Given the initial belief $b$, we can define the $n$-step forward kernel $H^{(n)}$ as follows where $b_{n+1}$ is the belief at time $n+1$:
		\begin{align}\label{def:n-forward-kernel}
		b_{n+1}=H^{(n)}(b,i_{1:n},r_{1:n}).
		\end{align}
		Then it is easy to see that the belief transition kernel $\bar{T}^{(n)}$ defined in \eqref{def:n-trans-ker} satisfy
		\begin{align}
		\bar{T}^{(n)}(\mathbf{A}|b,i_{1:n})=\int_{\mathcal{R}^n}\1_{\{H^{(n)}(b,i_{1:n},r_{1:n})\in \mathbf{A}\}}\bar{Q}^{(n)}(\prod_{t=1}^{n}dr_t|b,i_{1:n}).
		\end{align}
		Then we can obtain:
		{\small
			\begin{align}
			&\left|\int_{\mathcal{R}^n} f(b')\bar{T}^{(n)}(db'|b^1,i_{1:n})-\int_{\mathcal{R}^n} f(b')\bar{T}^{(n)}(db'|b^2,i_{1:n})\right|\\
			&=\left|\int_{\mathcal{R}^n} f(H^{(n)}(b^1,i_{1:n},r_{1:n}))\bar{Q}^{(n)}(\prod_{t=1}^{n}dr_t|b^1,i_{1:n})-\int_{\mathcal{R}^n} f(H^{(n)}(b^2,i_{1:n},r_{1:n}))\bar{Q}^{(n)}(\prod_{t=1}^{n}dr_t|b^2,i_{1:n})\right|\\
			&\leq\left|\int_{\mathcal{R}^n} f(H^{(n)}(b^1,i_{1:n},r_{1:n}))\left(\bar{Q}^{(n)}(\prod_{i=0}^{n-1}dr_t|b^1,i_{1:n})-\bar{Q}^{(n)}(\prod_{t=1}^{n}dr_t|b^2,i_{1:n})\right)\right|\\
			&\quad\quad+\left|\int_{\mathcal{R}^n} \left(f(H^{(n)}(b^1,i_{1:n},r_{1:n}))- f(H^{(n)}(b^2,i_{1:n},r_{1:n}))\right)\bar{Q}^{(n)}(\prod_{t=1}^{n}dr_t|b^2,i_{1:n})\right|.\label{bound-Kan-metric}
			\end{align}
		}
		We first bound the second term in \eqref{bound-Kan-metric}. We can infer from Theorem 3.7.1 in \cite{Krish2016} and its proof that the impact of initial belief decays exponentially fast:
		\begin{equation}
		|H^{(n)}(b^1,i_{1:n},r_{1:n})-H^{(n)}(b^2,i_{1:n},r_{1:n})|\leq C_4 \alpha^{n}||b^1-b^2||_1,
		\end{equation}
		where constant $C_4 = \frac{2(1-\epsilon)}{ \epsilon} $ and $\alpha= \frac{1- 2 \epsilon}{1 -\epsilon}< 1$.
		So the second term of \eqref{bound-Kan-metric} can be bounded by
		\begin{align}
		&\left|\int_{\mathcal{R}^n} \left(f(H^{(n)}(b^1,i_{1:n},r_{1:n}))- f(H^{(n)}(b^2,i_{1:n},r_{1:n}))\right)\bar{Q}^{(n)}(\prod_{t=1}^{n}dr_t|b^2,i_{1:n})\right|\\
		&\leq \left|\int_{\mathcal{R}^n} \left |H^{(n)}(b^1,i_{1:n},r_{1:n})- H^{(n)}(b^2,i_{1:n},r_{1:n})\right| \bar{Q}^{(n)}(\prod_{t=1}^{n}dr_t|b^2,i_{1:n})\right|\\
		&\leq C_4 \alpha^{n}||b^1-b^2||_1,\label{bound-Kan-metric-p2}
		\end{align}
		where the first inequality follows from $f\in \text{Lip}_1$. 

		It remains to bound the first term in \eqref{bound-Kan-metric}.
		Recall that $b$ defines the initial probability distribution $M_1$.
		The $n$ steps observation kernel is
		\begin{align}
		\bar{Q}^{(n)}(\prod_{t=1}^{n}dr_t|b,i_{1:n})&=\sum_{m\in\mathcal{M}}\Prob(M_1=m)\Prob(\prod_{t=1}^{n}dr_t|M_1=m,i_{1:n})=\sum_{m\in\mathcal{M}}b(m)\Prob(\prod_{t=1}^{n}dr_t|M_1=m,i_{1:n}).
		\end{align}
		Define a vector $g\in\R^M$ as:
		\begin{equation*}
		g(m) \coloneqq \int_{\mathcal{R}^n}f(H^{(n)}(b^1,i_{1:n},r_{1:n}))\Prob(\prod_{t=1}^{n}dr_t|M_1=m,i_{1:n}).
		\end{equation*}
		We can rewrite the first term of \eqref{bound-Kan-metric}:
		\begin{align}
		&\left|\int_{\mathcal{R}^n} f(H^{(n)}(b^1,i_{1:n},r_{1:n}))\left(\bar{Q}^{(n)}(\prod_{t=1}^{n}dr_t|b^1,i_{1:n})-\bar{Q}^{(n)}(\prod_{t=1}^{n}dr_t|b^2,i_{1:n})\right)\right|\\
		&=\left|\sum_{m=1}^{M}(b^1(m)-b^2(m))\int_{\mathcal{R}^n}f(H^{(n)}(b^1,i_{1:n},r_{1:n}))\Prob(\prod_{t=1}^{n}dr_t|M_1=m,i_{1:n})\right|\\
		&=\left|\sum_{m=1}^{M}\left(b^1(m)-b^2(m)\right)g(m)\right|\\
		&=\left|\sum_{m=1}^{M}\left(b^1(m)-b^2(m)\right)\left(g(m)-\frac{\max_m g(m)+\min_m g(m)}{2}\right)\right|\\
		&\leq \left \|b^1-b^2\right \|_1 \cdot \left \|g(m)-\frac{\max_m g(m)+\min_m g(m)}{2}\right \|_\infty\\
		&= \left \|b^1-b^2\right \|_1\frac{1}{2}\left(\max_m g(m)-\min_m g(m)\right),\label{bound-Kan-metric-p1-0}
		\end{align}
		where the next to last equality follows from $\sum_{m=1}^{M}\left(b^1(m)-b^2(m)\right)=0$.

		Next we bound $\max_m g(m)-\min_m g(m)$.
		From the equation above, it is clear that the quantity $\frac{1}{2}\left(\max_m g(m)-\min_m g(m)\right) \le 1$, because $||f ||_\infty \leq 1$. However to prove Lemma~\ref{lemma:lip_module_bound}, we need a sharper bound so that we can find a constant $C_5<1$ (that is independent of $b^1, n$ and $i_{1:n}$) with
		\begin{equation}\label{eq:C4-bound}
		\frac{1}{2}\left(\max_m g(m)-\min_m g(m)\right) \le C_5 <1.
		\end{equation}

		Suppose \eqref{eq:C4-bound} holds. Then
		on combining \eqref{bound-Kan-metric}, \eqref{bound-Kan-metric-p2} and \eqref{bound-Kan-metric-p1-0}, we obtain
		\begin{align}\label{inequ:Kan_insider}
		\left|\int_{\mathcal{B}} f(b')\bar{T}^{(n)}(db'|b^1,i_{1:n})-\int_{\mathcal{B}} f(b')\bar{T}^{(n)}(db'|b^2,i_{1:n})\right|
		\leq C_4 \alpha^{n}||b^1-b^2||_1+C_5 ||b^1-b^2||_1.
		\end{align}
		It then follows that the Kantorovich metric is bounded by
		\begin{align}\label{inequ:Kantorovich metric}
		K\left(\bar{T}^{(n)}(db'|b^1,i_{1:n}),\bar{T}^{(n)}(db'|b^2,i_{1:n})\right) \leq C_4\alpha^{n}||b^1-b^2||_1+C_5||b^1-b^2||_1,
		\end{align}
		where $C_4=\frac{2(1-\epsilon)}{\epsilon}, \alpha=1-\frac{\epsilon}{1-\epsilon}$, and $\epsilon=\min\limits_{m,m' \in \mathcal{M}}P_{m,m'}>0$.
		So $\bar{T}^{(n)}$ is Lipschitz uniformly in actions, and
		its Lipschitz module can be bounded as follows:
		\begin{align}
		l^{\mathcal I^{n}}_{\rho_{\mathcal{B}}}(\bar{T}^{(n)}):=\sup_{i_{1:n}}\sup_{b^1\neq b^2}\frac{K\left(\bar{T}^{(n)}(db'|b^1,i_{1:n}),\bar{T}^{(n)}(db'|b^2,i_{1:n})\right)}{\rho_{\mathcal{B}}(b^1,b^2)} \le C_4\alpha^{n}+C_5.
		\end{align}
		If we choose $n=n_0:=\lceil \log_{\alpha}\frac{1-C_5}{2C_4} \rceil$, so that
		$C_4\alpha^{n_0}+C_5<\frac{1}{2}(1+C_5)\coloneqq\gamma<1$, then
		we obtain the desired result $
		l^{\mathcal I^{n_0}}_{\rho_{\mathcal{B}}}(\bar{T}^{(n_0)})< \gamma.$

		It remains to prove \eqref{eq:C4-bound}.
		Since the set $\mathcal{M}=\{1, \ldots, M\}$ is finite, we pick $m^* \in \argmin \limits_{m \in \mathcal{M}} g(m), \hat m \in \argmax \limits_{m \in \mathcal{M}} g(m)$. We have
		\begin{align} \label{eq:diff-g}
		&\frac{1}{2}\left(\max_m g(m)-\min_m g(m)\right)\\
		&= \frac{1}{2}\sum_{r_{1:n} \in \mathcal{R}^n}f(H^{(n)}(b^1,i_{1:n},r_{1:n}))\left(\Prob(r_{1:n}|M_1=\hat m,i_{1:n})-\Prob(r_{1:n}|M_1=m^*,i_{1:n})\right)\\
		&\leq \frac{1}{2} \sum_{r_{1:n} \in \mathcal{R}^n } \left|\Prob(r_{1:n}|M_1=\hat m,i_{1:n})-\Prob(r_{1:n}|M_1=m^*,i_{1:n})\right|,
		\end{align}
		where the inequality follows from H\" older's inequality with $||f||_\infty\leq 1$.
		We can compute
		\begin{align}
		&\Prob(r_{1:n}|M_1=m_1,i_{1:n})\\
		&=\sum_{m_{2:n}\in\mathcal{M}^{n-1}}\Prob(r_{1:n}|M_1=m_1,i_{1:n},M_{2:n}=m_{2:n})\cdot \Prob(M_{2:n}=m_{2:n}|M_1=m_1,i_{1:n})\\
		&=\sum_{m_{2:n}\in\mathcal{M}^{n-1}}\Prob(r_{1:n}|i_{1:n},m_{1:n})\cdot \Prob(m_{2:n}|M_1=m_1,i_{1:n})\\
		&=\sum_{m_{2:n}\in\mathcal{M}^{n-1}}\left(\prod_{t=1}^{n}\Prob(r_t|m_t,i_{t})\right) \cdot \left(\prod_{t=1}^{n-1}\Prob(m_{t+1}|m_t,i_{t})\right),
		\end{align}
		where the last equality holds due to the conditional independence. We can then infer that for any $\{r_{1:n}\},  \{i_{1:n}\}$,
		\begin{align}
		&\Prob(r_{1:n}|M_1=m^*,i_{1:n})\\
		&=\sum_{m_{2:n}\in\mathcal{M}^{n-1}}\left(\prod_{t=2}^{n}\Prob(r_t|m_t,i_t)\right) \cdot \left(\prod_{t=1}^{n-1}\Prob(m_{t+1}|m_t,i_{t})\right) \cdot P(m^*,m_2) \\
		& \ge \sum_{m_{2:n}\in\mathcal{M}^{n-1}}\left(\prod_{t=1}^{n}\Prob(r_t|m_t,i_{t})\right) \cdot \left(\prod_{t=2}^{n-1}\Prob(m_{t+1}|m_t,i_{t})\right) \cdot P( \hat m,m_1) \cdot \frac{\epsilon}{1 - \epsilon} \\
		& = \Prob(r_{1:n}|M_1= \hat m,i_{1:n}) \cdot \frac{\epsilon}{1 - \epsilon}.
		\end{align}
		It follows that
		\begin{align*}
		& \left|\Prob(r_{1:n}|M_1=\hat m,i_{1:n})-\Prob(r_{1:n}|M_1=m^*,i_{1:n})\right| \\
		& \le \max\left\{\left(1- \frac{\epsilon}{1 - \epsilon} \right) \Prob(r_{1:n}|M_1=\hat m,i_{1:n}), \Prob(r_{1:n}|M_1=m^*,i_{1:n}) \right\} \\
		& \le \left(1- \frac{\epsilon}{1 - \epsilon} \right) \Prob(r_{1:n}|M_1=\hat m,i_{1:n}) + \Prob(r_{1:n}|M_1=m^*,i_{1:n}).
		\end{align*}
		Then we can obtain from \eqref{eq:diff-g} that
		\begin{align}
		&\frac{1}{2}\left(\max_m g(m)-\min_m g(m)\right)\\
		& \leq \frac{1}{2} \sum_{r_{1:n} \in \mathcal{R}^n} \left[
		\left(1- \frac{\epsilon}{1 - \epsilon} \right) \Prob(r_{1:n}|M_0=\hat m,i_{1:n}) + \Prob(r_{1:n}|M_0=m^*,i_{1:n}) \right] \\
		& = \alpha/2 + 1/2 := C_5 <1,
		\end{align}
		where $\alpha= 1 - \frac{\epsilon}{1 - \epsilon} \in (0,1).$ The proof is complete.
	\end{proof}

	\section{Proof of Theorem \ref{thm:upper_bound}} \label{sec:proof-complete}

	In this section we provide a complete version of the proof of Theorem~\ref{thm:upper_bound} sketched in Section~\ref{sec:proof1}.

	We follow our prior simplication that random reward follows the Bernoulli distribution. Here we want to clarify two notations in advance. $\E^{\pi}$ means the expectation is taken respect to true mean reward matrix $\bm{\mu}$ and transition probabilities $P$, and $\E_{k}^{\pi}$ denotes that the underlying parameters are estimators $\bm{\mu}_k$ and $P_k$.

	Recalling the definition of regret in \eqref{def:reg}, it can be rewritten as:
	\begin{align}
	\mathcal{R}_T=\sum_{t=1}^{T}(\rho^*-R_t)=\sum_{t=1}^{T}(\rho^*-\E^\pi[R_t|\mathcal{F}_{t-1}])+\sum_{t=1}^{T}(\E^\pi[R_t|\mathcal{F}_{t-1}]-R_t).\label{reg}
	\end{align}

	We first bound the second term of \eqref{reg}, i.e. the total bias between the conditional expectation of reward and the realization. Define a stochastic process $\{X_n, n=0,\cdots,T\}$ as:
	\begin{align}
	X_0=0,\quad
	X_t=\sum_{l=1}^{t}(\E^\pi[R_l|\mathcal{F}_{l-1}]-R_l),
	\end{align}
	then the second term in \eqref{reg} is $X_T$.
	It is easy to see that $X_t$ is a martingale. Moreover, due to the Bernoulli distribution of $R_t$,
	\begin{align}
	|X_{t+1}-X_t|=|\E^{\pi}[R_{t+1}|\mathcal{F}_t]-R_{t+1}|\leq 1.
	\end{align}
	Applying the Azuma-Hoeffding inequality \cite{Azuma1967}, we have
	\begin{align}\label{reg:expreward-reward}
	\Prob\left(\sum_{t=1}^{T}(\E^\pi[R_t|\mathcal{F}_{t-1}]-R_t) \geq \sqrt{2T\ln{\frac{1}{\delta}}}\right) \leq \delta.
	\end{align}

	Next we bound the first term of  \eqref{reg}.
	Recall that the definition of belief state under the optimistic and true parameters:
	$b_t^k(m)=\Prob_{\bm{\mu}_k,P_k}(M_t=m|\mathcal{F}_{t-1})$ and
	$b_t(m)=\Prob(M_t=m|\mathcal{F}_{t-1})$, and the definition of reward functions with respect to the true belief state
	$\bar{c}(b_t,i)=\sum_{m=1}^{M}\mu_{m,i}b_t(m)$.
	We can also define the reward functions with respect to the optimistic belief state $b_t^k$ as:
	\begin{align}
	\bar{c}_k(b_t^k,i)=\sum_{m=1}^{M}(\bm{\mu}_k)_{m,i} b_t^k(m).
	\end{align}

	Because $I_t$ is also adapted to $\mathcal F_{t-1}$, we have
	\begin{align}
	&\E^{\pi}[\mu (M_t, I_t)|\mathcal{F}_{t-1}]=\bar{c}(b_t,I_t)=\langle (\bm{\mu})_{I_t},b_t \rangle,\\
	&\E_{k}^{\pi}[\mu_k (M_t, I_t)|\mathcal{F}_{t-1}]=\bar{c}_k(b_t^k,I_t)=\langle (\bm{\mu}_k)_{I_t},b_t^k \rangle,\label{equ:belief-reward-func}
	\end{align}
	where $\mu (M_t, I_t),\mu_k (M_t, I_t)$ are the $M_t$-th row $I_t$-th column element of matrix $\mu,\mu_k$ respectively, and  $(\bm{\mu}_k)_{I_t}$ and $(\bm{\mu})_{I_t}$ are the $I_t$-th column vector of the reward matrix $\bm{\mu}_k$ and $\bm{\mu}$, respectively.

	Then we can rewrite the first term of \eqref{reg}:
	\begin{align}\label{reg1}
	\sum_{t=1}^{T}(\rho^*-\E^\pi[R_t|\mathcal{F}_{t-1}])=\sum_{t=1}^{T}(\rho^*-\E^\pi[\mu(M_t, I_t)|\mathcal{F}_{t-1}])=\sum_{t=1}^{T}(\rho^*-\bar{c}(b_t,I_t)),
	\end{align}
	where the first equation is due to the tower property and the fact that $R_t$ and $\mathcal{F}_{t-1}$ are conditionally independent given $M_t$ and $I_t$.

	Let $K$ be the number of total episodes. For each episode $k=1, 2, \cdots, K$, let $H_k, E_k$ be the exploration and exploitation phases, respectively. Then we can split equation \eqref{reg1} to the summation of the bias in these two phases as:
	\begin{align}\label{reg2}
	\sum_{k=1}^{K}\sum_{t\in H_k}(\rho^*-\bar{c}(b_t,I_t))+\sum_{k=1}^{K}\sum_{t\in E_k}(\rho^*-\bar{c}(b_t,I_t)).
	\end{align}
	Moreover, we remark here that the length of the last exploitation phase $|E_K|=\min\{\tau_2\sqrt{K},\max\{T-(K\tau_1+\sum_{k=1}^{K-1}\tau_2\sqrt{k}),0\}\}$, as it may end at period $T$.

	{\it Step 1: Bounding the regret in exploration phases}

	The first term of \eqref{reg2} can be simply upper bounded by:
	\begin{align}\label{reg:explore-bound}
	\sum_{k=1}^{K}\sum_{t\in H_k}(\rho^*-\bar{c}(b_t,I_t)) \leq\sum_{k=1}^{K}\sum_{t\in H_k}\rho^*=K\tau_1\rho^*.
	\end{align}

	{\it Step 2: Bounding the regret in exploitation phases}

	We bound it by separating into ``success'' and ``failure'' events below.
	Recall that in episode $k$, we define the set of plausible POMDPs $\mathbb{G}_k(\delta_k)$, which is defined in terms of confidence regions $\mathcal{C}_k(\delta_k)$ around the estimated mean reward matrix $\bm{\mu}_k$ and the transition probabilities $P_k$. Then choose an optimistic POMDP $\widetilde{\mathbb{G}}_k \in \mathbb{G}_k(\delta_k)$ that has the optimal average reward among the plausible POMDPs and denote its corresponding reward matrix, value function, and the optimal average reward by $\bm\mu_k,v_k$ and $\rho^k$, respectively. Thus, we say a ``success'' event if and only if the set of plausible POMDPs $\mathbb{G}_k(\delta_k)$ contains the true POMDP $\mathbb{G}$. In the following proof, we omit the dependence on $\delta_k$ from $\mathbb{G}_k(\delta_k)$ for simplicity.
	From Algorithm \ref{alg:SEEU}, the confidence level of $\mu_k$ in episode $k$ is $1-\delta_k$, we can obtain:
	\begin{align}
	\Prob(\mathbb{G}\notin\mathbb{G}_k, \text{for some}~k)
	\leq \sum_{k=1}^{K}\delta_k=\sum_{k=1}^{K}\frac{\delta}{k^3}\leq \frac{3}{2}\delta.
	\end{align}

	Thus, with probability at least $1-\frac{3}{2}\delta$, ``success'' events happen. It means $\rho^*\leq \rho^k$ for any $k$ because $\rho^k$ is the optimal average reward of the optimistic POMDP $\widetilde{\mathbb{G}}_k$ from the set $\mathbb{G}_k$. Then we can bound the regret of ``success'' events in exploitation phases as follows.
	\begin{align}
	\sum_{k=1}^{K}\sum_{t\in E_k}(\rho^*-\bar{c}(b_t,I_t))
	& \leq \sum_{k=1}^{K}\sum_{t\in E_k}(\rho^k-\bar{c}(b_t,I_t)) \\
	&=\sum_{k=1}^{K}\sum_{t\in E_k}(\rho^k-\bar{c}_k(b_t^k,I_t))+(\bar{c}_k(b_t^k,I_t)-\bar{c}(b_t,I_t)).\label{reg:success-bound-v1}
	\end{align}

	To bound the first term of formula \eqref{reg:success-bound-v1}, we use the Bellman optimality equation for the optimistic belief MDP $\widetilde{\mathbb{G}}_k$ on the continuous belief state space $\mathcal{B}$:
	\begin{align}\label{equ:Bellman-k}
	\rho^k+v_k(b_t^k)= \bar{c}_k(b_t^k,I_t)+\int_{b_{t+1}^k\in\mathcal{B}}v_k(b_{t+1}^k)\bar{T}_{k}(db_{t+1}^k|b_t^k,I_t)
	=\bar{c}_k(b_t^k,I_t)+\langle \bar{T}_{k}(\cdot|b_t^k,I_t),v_k(\cdot) \rangle,
	\end{align}
	where $\bar{T}_{k}(\cdot|b_t^k,I_t)=\mathbb{P}_{\bm{\mu}_k, P_k}(b_{t+1}\in \cdot|b_t^k,I_t)$ means transition probability of the belief state conditional on pulled arm under estimated reward matrix $\bm{\mu}_k$ and transition matrix of underlying Markov chain $P_k$ at time $t$.

	Moreover, we note that if value function $v_k$ satisfies the Bellman equation \eqref{equ:Bellman-thm}, then so is $v_k+c\bm{1}$.
	Thus, without loss of generality, we assume that $v_k$ needs to satisfy $||v_k||_{\infty}\leq \text{span}(v_k)/2$. Then from Proposition \ref{prop:span-uni-bound}, suppose the uniform bound of $\text{span}(v_k)$ is $D$, then we have:
	\begin{align}\label{vk-infty-norm-bound}
	||v_k||_{\infty}\leq \frac{1}{2}\text{span}(v_k)\leq \frac{D}{2}.
	\end{align}
	Thus the first term of \eqref{reg:success-bound-v1} can be bounded by
	\begin{align}
	&\sum_{k=1}^{K}\sum_{t\in E_k}(\rho^k-\bar{c}_k(b_t^k,I_t))\\
	&=\sum_{k=1}^{K}\sum_{t\in E_k}(-v_k(b_t^k)+\langle \bar{T}_{k}(\cdot|b_t^k,I_t),v_k(\cdot)\rangle)\\
	&=\sum_{k=1}^{K}\sum_{t\in E_k}(-v_k(b_t^k)+\langle \bar{T}(\cdot|b_t^k,I_t),v_k(\cdot)\rangle)+\langle \bar{T}_{k}(\cdot|b_t^k,I_t)-\bar{T}(\cdot|b_t^k,I_t),v_k(\cdot)\rangle,\label{reg:success-bound-p1}
	\end{align}
	where recall that $\bar{T}_{k}(\cdot|b_t^k,I_t)$ and $\bar{T}(\cdot|b_t^k,I_t)$ in the second equality are the belief state transition probabilities under estimated and true parameters, respectively.
	And the last inequality is applying the H$\ddot{\text{o}}$lder's inequality.


	For the first term of \eqref{reg:success-bound-p1}, we have
	\begin{align}
	&\sum_{k=1}^{K}\sum_{t\in E_k}(-v_k(b_t^k)+\langle \bar{T}(\cdot|b_t^k,I_t),v_k(\cdot)\rangle)\\
	&=\sum_{k=1}^{K}\sum_{t\in E_k}(-v_k(b_t^k)+v_k(b_{t+1}^k))+(-v_k(b_{t+1}^k)+\langle \bar{T}(\cdot|b_t^k,I_t),v_k(\cdot)\rangle)\\
	&=\sum_{k=1}^{K}v_k(b_{t_k+\tau_1+\tau_2\sqrt{k}}^k)-v_k(b_{t_k+\tau_1+1}^k)+\sum_{k=1}^{K}\sum_{t\in E_k}\E^\pi[v_k(b_{t+1}^k)|\mathcal{F}_t]-v_k(b_{t+1}^k).
	\end{align}
	where the first term in the last equality is due to the telescoping from $t_k+\tau_1+1$ to $t_k+\tau_1+\tau_2 \sqrt{k}$, the start and end of the exploitation phase in episode $k$.
	The second term in the last equality is because:
	\begin{align}
	\langle \bar{T}(\cdot|b_t^k,I_t),v_k(\cdot)\rangle=\int_{b_{t+1}^k\in\mathcal{B}}v_k(b_{t+1}^k)\bar{T}(db_{t+1}^k|b_t^k,I_t)=\E^\pi[v_k(b_{t+1}^k)|b_t^k]=\E^\pi[v_k(b_{t+1}^k)|\mathcal{F}_t].
	\end{align}
	Applying Proposition \ref{prop:span-uni-bound}, 	we have
	\begin{align}
	v_k(b_{t_k+\tau_1+\tau_2\sqrt{k}}^k)-v_k(b_{t_k+\tau_1+1}^k)\leq D.
	\end{align}
	We also need the following result, the proof of which is deferred to the end of this section.
	\begin{proposition}\label{prop:Azuma}
		Let $K$ be the number of total episodes up to time $T$. For each episode $k=1, \cdots, K$, let $E_k$ be the index set of the $k$th exploitation phase and $v_k$ be the value function of the optimistic POMDP at the $k$th exploitation phase.
		Then with probability at most $\delta$,
		\begin{align}\label{equ:Azuma}
		\sum_{k=1}^{K}\sum_{t\in E_k}\mathbb{E}^\pi[v_k(b_{t+1})|\mathcal{F}_t]-v_k(b_{t+1}) \geq D \sqrt{2T \ln(\frac{1}{\delta})},
		\end{align}
	\end{proposition}
	where the expectation $\mathbb{E}^{\pi}$ is taken respect to the true parameters $\bm{\mu}$ and $P$ under policy $\pi$, and the filtration $\mathcal{F}_t$ is defined as $\mathcal{F}_t \coloneqq \sigma(\pi_1,R^{\pi}_1,...,\pi_{t-1},R^{\pi}_{t-1})$.

	Applying Proposition \ref{prop:Azuma}, with probability at least $1-\delta$, we have:
	\begin{align}
	\sum_{k=1}^{K}\sum_{t\in E_k}\E^\pi[v_k(b_{t+1}^k)|\mathcal{F}_t]-v_k(b_{t+1}^k)\leq D \sqrt{2T \ln(\frac{1}{\delta})}.
	\end{align}
	Thus, the first term of \eqref{reg:success-bound-p1} can be upper bounded by:
	\begin{align}\label{reg:success-bound-p1-p1}
	KD+D \sqrt{2T \ln(\frac{1}{\delta})}.
	\end{align}

	For the second term of \eqref{reg:success-bound-p1}, we note that  $\bar{T}(b_{t+1}^k|b_t^k,I_t)$ is zero except for two points where $b_{t+1}^k$ are exactly the Bayesian updating after receiving an observation of Bernoulli reward $r_t$ taking value $0$ or $1$. Thus, we have the following transition kernel:

	\begin{align}
	&\langle \bar{T}_{k}(\cdot|b_t^k,I_t)-\bar{T}(\cdot|b_t^k,I_t),v_k(\cdot)\rangle\\
	&\leq \left|\int_{\mathcal{B}} v_k(b')\bar{T}_k(db'|b_t^k,I_t)-\int_{\mathcal{B}} v_k(b')\bar{T}(db'|b_t^k,I_t)\right|\\
	&=\left| \sum_{r_{t} \in \mathcal{R}} v_k\left(H_{k}\left(b_{t}^{k}, I_{t}, r_{t}\right) \right)\mathbb{P}_k\left(r_{t}| b_{t}^{k}, I_{t}\right)-\sum_{r_{t} \in \mathcal{R}} v_k\left(H\left(b_{t}^{k}, I_{t}, r_{t}\right) \right)\mathbb{P}\left(r_{t}| b_{t}^{k}, I_{t}\right) \right|\\
	& \leq \left| \sum_{r_{t} \in \mathcal{R}} v_k\left(H_{k}\left(b_{t}^{k}, I_{t}, r_{t}\right) \right) \cdot  \left[\mathbb{P}_k\left(r_{t}| b_{t}^{k}, I_{t}\right)- \mathbb{P}\left(r_{t}| b_{t}^{k}, I_{t}\right) \right] \right|\\
	&\quad \quad + \left|\sum_{r_{t} \in \mathcal{R}} \left[ v_k\left(H_{k}\left(b_{t}^{k}, I_{t}, r_{t}\right) \right)-v_k\left(H\left(b_{t}^{k}, I_{t}, r_{t}\right) \right)\right]\cdot \mathbb{P}\left(r_{t}| b_{t}^{k}, I_{t}\right)\right|,\label{reg:belief-transition-distance}
	\end{align}
	where we use $H_k$ and $H$ to denote the belief updating function under the optimistic model $(\bm \mu_k, P_k)$ and the true model $(\bm \mu, P)$, and we use $\mathbb{P}_k$ and $\mathbb{P}$ to denote the probability with respect to the optimistic model and true model respectively.

	We bound the first term of \eqref{reg:belief-transition-distance} by
	\begin{align}
	&\left| \sum_{r_{t} \in \mathcal{R}} v_k\left(H_{k}\left(b_{t}^{k}, I_{t}, r_{t}\right) \right) \cdot  \left[\mathbb{P}_k\left(r_{t}| b_{t}^{k}, I_{t}\right)- \mathbb{P}\left(r_{t}| b_{t}^{k}, I_{t}\right) \right] \right|\\
	& \leq \left| v_k\left(H_{k}\left(b_{t}^{k}, I_{t}, r_{t}=1\right) \right) \cdot  \left[\mathbb{P}_k\left(r_{t}=1| b_{t}^{k}, I_{t}\right)- \mathbb{P}\left(r_{t}=1| b_{t}^{k}, I_{t}\right) \right] \right|\\
	&\quad\quad + \left| v_k\left(H_{k}\left(b_{t}^{k}, I_{t}, r_{t}=0\right) \right) \cdot  \left[\mathbb{P}_k\left(r_{t}=0| b_{t}^{k}, I_{t}\right)- \mathbb{P}\left(r_{t}=0| b_{t}^{k}, I_{t}\right) \right] \right|\\
	& =\left| v_k\left(H_{k}\left(b_{t}^{k}, I_{t}, 1\right) \right) \cdot \left[\langle (\bm{\mu}_k)_{I_t},b_t^k\rangle-\langle (\bm{\mu})_{I_t},b_t^k\rangle\right]\right|\\
	&\quad\quad+\left| v_k\left(H_{k}\left(b_{t}^{k}, I_{t}, 0\right) \right) \cdot \left[1-\langle (\bm{\mu}_k)_{I_t},b_t^k\rangle-\left(1-\langle (\bm{\mu})_{I_t},b_t^k\rangle\right)\right] \right|\\
	& \leq 2||v_k||_\infty \cdot \left|\langle (\bm{\mu}_k)_{I_t},b_t^k\rangle-\langle (\bm{\mu})_{I_t},b_t^k\rangle\right|\\
	& \leq D \left|\langle (\bm{\mu}_k)_{I_t},b_t^k\rangle-\langle (\bm{\mu})_{I_t},b_t^k\rangle\right|\\
	&\leq D || (\bm{\mu}_k)_{I_t} -(\bm{\mu})_{I_t}||_1 \cdot ||b_t^k||_\infty\\
	&\leq D|| (\bm{\mu}_k)_{I_t} -(\bm{\mu})_{I_t}||_1, \label{reg:belief-transition-distance-p1}
	\end{align}
	where the first equality comes from $\mathbb{P}_k\left(r_{t}=1| b_{t}^{k}, I_{t}\right)=\sum\limits_{m\in\mathcal{M}} \Prob_{k}(r_t=1|m,I_t) b_t^k(m)=\langle (\bm{\mu})_{I_t},b_t^k\rangle$ and $\mathbb{P}_k\left(r_{t}=0| b_{t}^{k}, I_{t}\right)=\sum\limits_{m\in\mathcal{M}} \Prob_{k}(r_t=0|m,I_t) b_t^k(m)=1-\langle (\bm{\mu})_{I_t},b_t^k\rangle$,
	and the third inequality is from \eqref{vk-infty-norm-bound}, we have $||v_k||_\infty \leq \frac{D}{2}$.

	We can bound the second term of \eqref{reg:belief-transition-distance} as follows:
	\begin{align}
	&\left|\sum_{r_{t} \in \mathcal{R}} \left[ v_k\left(H_{k}\left(b_{t}^{k}, I_{t}, r_{t}\right) \right)-v_k\left(H\left(b_{t}^{k}, I_{t}, r_{t}\right) \right)\right]\cdot \mathbb{P}\left(r_{t}| b_{t}^{k}, I_{t}\right)\right|\\
	&\leq \sum_{r_{t} \in \mathcal{R}} \left| v_k\left(H_{k}\left(b_{t}^{k}, I_{t}, r_{t}\right) \right)-v_k\left(H\left(b_{t}^{k}, I_{t}, r_{t}\right) \right)\right| \cdot \mathbb{P}\left(r_{t}| b_{t}^{k}, I_{t}\right)\\
	&\leq \sum_{r_{t} \in \mathcal{R}} \frac{D}{2} \left| H_{k}\left(b_{t}^{k}, I_{t}, r_{t} \right)-H\left(b_{t}^{k}, I_{t}, r_{t}\right) \right| \cdot \mathbb{P}\left(r_{t}| b_{t}^{k}, I_{t}\right)\\
	& \leq \sum_{r_{t} \in \mathcal{R}} \frac{D}{2} \left[ L_1|| \bm{\mu} -\bm{\mu}_k||_1+L_2||P- P_k||_F \right] \mathbb{P}\left(r_{t}| b_{t}^{k}, I_{t}\right)\\
	& = \frac{D}{2} \left[ L_1|| \bm{\mu} -\bm{\mu}_k||_1+L_2||P- P_k||_F \right], \label{reg:belief-transition-distance-p2}
	\end{align}
	where the second inequality is implied from the proof of Proposition \ref{prop:span-uni-bound}, and the last inequality if from Proposition \ref{prop:lip_bt}.

	Therefore, from \eqref{reg:belief-transition-distance-p1} and \eqref{reg:belief-transition-distance-p2}, we can obtain that the second term of equation \eqref{reg:success-bound-p1}:


	\begin{align}
	&\sum_{k=1}^{K}\sum_{t\in E_k}\langle \bar{T}_{k}(\cdot|b_t^k,I_t)-\bar{T}(\cdot|b_t^k,I_t),v_k(\cdot)\rangle\\
	& \leq \sum_{k=1}^{K}\sum_{t\in E_k} D\left[ || (\bm{\mu})_{I_t} -(\bm{\mu}_k)_{I_t}||_1+\frac{L_1}{2} || \bm{\mu} -\bm{\mu}_k||_1+\frac{L_2}{2}||P-  P_k||_F\right].\label{reg:success-bound-p1-p2}
	\end{align}

	Summing up \eqref{reg:success-bound-p1-p1} and \eqref{reg:success-bound-p1-p2}, the first term of formula \eqref{reg:success-bound-v1} can be bounded by
	\begin{align}
	&\sum_{k=1}^{K}\sum_{t\in E_k}(\rho^k-\bar{c}_k(b_t^k,I_t)) \leq KD+D \sqrt{2T \ln(\frac{1}{\delta})}\\
	&\quad\quad+\sum_{k=1}^{K}\sum_{t\in E_k} D\left[ || (\bm{\mu})_{I_t} -(\bm{\mu}_k)_{I_t}||_1+\frac{L_1}{2} || \bm{\mu} -\bm{\mu}_k||_1+\frac{L_2}{2}||P-  P_k||_F\right].\label{reg:success-bound-p1-v2}
	\end{align}

	Next we proceed to bound the second term of \eqref{reg:success-bound-v1}. By  \eqref{equ:belief-reward-func}, it can be rewritten as
	\begin{align}
	&\sum_{k=1}^{K}\sum_{t\in E_k}\bar{c}_k(b_t^k,I_t)-\bar{c}(b_t,I_t)\\
	&=\sum_{k=1}^{K}\sum_{t\in E_k}\langle (\bm{\mu}_k)_{I_t},b_t^k \rangle-\langle (\bm{\mu})_{I_t},b_t \rangle\\
	&=\sum_{k=1}^{K}\sum_{t\in E_k}\langle (\bm{\mu}_k)_{I_t},b_t^k \rangle- \langle (\bm{\mu})_{I_t},b_t^k \rangle+\langle (\bm{\mu})_{I_t},b_t^k \rangle-\langle (\bm{\mu})_{I_t},b_t \rangle,
	\end{align}
	then from the H$\ddot{\text{o}}$lder's inequality, we can further bounded the right hand side of above formula to
	\begin{align}
	&\sum_{k=1}^{K}\sum_{t\in E_k} ||(\bm{\mu}_k)_{I_t}-(\bm{\mu})_{I_t}||_1||b_t^k||_\infty+||(\bm{\mu})_{I_t}||_\infty||b_t^k-b_t||_1\\
	&\leq\sum_{k=1}^{K}\sum_{t\in E_k} ||(\bm{\mu}_k)_{I_t}-(\bm{\mu})_{I_t}||_1+||b_t^k-b_t||_1.\label{reg:success-bound-p2}
	\end{align}

	By Proposition \ref{prop:lip_bt}, we have $||b_t^k-b_t||_1\leq L_1||\bm{\mu}-\bm{\mu}_k||_1+L_2||P-P_k||_F$.
	Combining the above expression with \eqref{reg:success-bound-p1-v2}, we can see that with probability $1-\delta$, the regret incurred from ``success'' events, i.e., \eqref{reg:success-bound-v1}, can be bounded
	\begin{align}
	&\sum_{k=1}^{K}\sum_{t\in E_k}(\rho^*-\bar{c}(b_t,I_t))\\
	&\leq KD+D \sqrt{2T \ln(\frac{1}{\delta})}+\sum_{k=1}^{K}\sum_{t\in E_k}(D+1)||(\bm{\mu})_{I_t}-(\bm{\mu}_k)_{I_t}||_1 \\
	&\quad\quad +\sum_{k=1}^{K}\sum_{t\in E_k}\left[\left(1+\frac{D}{2}\right)L_1||\bm{\mu}-\bm{\mu}_k||_1+\left(1+\frac{D}{2}\right)L_2||P-P_k||_F\right].
	\end{align}

	Let $T_0$ be the period that the number of samples collected in the exploration phases exceeds $N_0$, that is,
	\begin{align}
	T_0:=\inf\limits_{t\geq 1}\{\sum_{n=1}^t\1_{(n\in H_k,\text{ for some }k)}\geq N_0\}\label{eqn:T_0}.
	\end{align}
	If $T\geq T_0$ after episode $k_0$, then from Proposition \ref{prop:spectral}, under the confidence level $\delta_k=\delta/k^3$, we have
	\begin{align}
	&||(\bm{\mu}_k)^m-(\bm{\mu})^m||_2\leq C_1\sqrt{\frac{\log(\frac{6(S^2+S)k^3}{\delta})}{\tau_1 k}}, \quad m \in \mathcal{M},\\
	&||P_k-P||_2\leq C_2\sqrt{\frac{\log(\frac{6(S^2+S)k^3}{\delta})}{\tau_1 k}}.
	\end{align}
	Together with the fact that the vector norm and matrix norm satisfy
	\begin{align}
	&||(\bm{\mu}_k)_i-(\bm{\mu})_i||_1 \leq\sum_{m=1}^M||(\bm{\mu}_k)^m-(\bm{\mu})^m||_1\leq \sum_{m=1}^M\sqrt{M}||(\bm{\mu}_k)^m-(\bm{\mu})^m||_2, \quad i \in \mathcal{I},\\
	&||\bm{\mu}_k-\bm{\mu}||_1 =\max_{i} ||(\bm{\mu}_k)_i-(\bm{\mu})_i||_1\leq\sum_{m=1}^M||(\bm{\mu}_k)^m-(\bm{\mu})^m||_1\leq \sum_{m=1}^M\sqrt{M}||(\bm{\mu}_k)^m-(\bm{\mu})^m||_2, \\
	&||P-P_k||_F\leq \sqrt{M} ||P-P_k||_2,
	\end{align}
	we obtain with probability at least $1-\frac{5}{2}\delta$, the regret in exploitation phase can be bounded by
	\begin{align}
	&\sum_{k=1}^{K}\sum_{t\in E_k}\rho^*-\bar{c}(b_t,I_t)\\
	&\leq KD+D \sqrt{2T \ln(\frac{1}{\delta})}+\sum_{k=k_0}^{K}\tau_2\sqrt{k}\left(D+1+\left(1+\frac{D}{2}\right)L_1\right)M\sqrt{M}C_1\sqrt{\frac{\log(\frac{6(S^2+S)k^3}{\delta})}{\tau_1 k}}\\
	&\quad\quad+\sum_{k=k_0}^{K}\tau_2\sqrt{k}\left(1+\frac{D}{2}\right)L_2 \sqrt{M}C_2\sqrt{\frac{\log(\frac{6(S^2+S))k^3}{\delta})}{\tau_1 k}}+T_0\rho^*\\
	&\leq KD+D \sqrt{2T \ln(\frac{1}{\delta})}+T_0\rho^*\\
	&\quad\quad+K\tau_2\left[\left(D+1+\left(1+\frac{D}{2}\right)L_1\right)M^{3/2}C_1+ \left(1+\frac{D}{2}\right)L_2M^{1/2}C_2\right]\sqrt{\frac{\log(\frac{6(S^2+S)K^3}{\delta})}{\tau_1}}.\\\label{reg:success-bound-v3}
	\end{align}

	{\it Step 3: Summing up the regret}

	Combining \eqref{reg:explore-bound} and \eqref{reg:success-bound-v3}, we can get that with probability at least $1-\frac{5}{2}\delta$, the first term of regret \eqref{reg} is bounded by
	\begin{align}
	&\sum_{t=1}^{T}\rho^*-\bar{c}(b_t,I_t)\\
	&\leq (K\tau_1+T_0)\rho^*+KD+D \sqrt{2T \ln(\frac{1}{\delta})}\\
	&\quad\quad+K\tau_2\left[\left(D+1+\left(1+\frac{D}{2}\right)L_1\right)M^{3/2}C_1+ \left(1+\frac{D}{2}\right)L_2M^{1/2}C_2\right]\sqrt{\frac{\log(\frac{6(S^2+S)K^3}{\delta})}{\tau_1}}.\label{reg:filtration}
	\end{align}

	Finally, combining \eqref{reg:expreward-reward} and \eqref{reg:filtration}, we can see that with probability at least $1-\frac{7}{2}\delta$, the regret presented in \eqref{reg} can be bounded by
	\begin{align}
	\mathcal{R}_T
	&\leq (K\tau_1+T_0)\rho^*+KD+D \sqrt{2T \ln(\frac{1}{\delta})}+\sqrt{2T\ln{\frac{1}{\delta}}}\\
	&\quad\quad+K\tau_2\left[\left(D+1+\left(1+\frac{D}{2}\right)L_1\right)M^{3/2}C_1+ \left(1+\frac{D}{2}\right)L_2M^{1/2}C_2\right]\sqrt{\frac{\log(\frac{6(S^2+S)K^3}{\delta})}{\tau_1}}).
	\end{align}
	Note that
	\begin{align}
	\sum_{k=1}^{K-1}\tau_1+\tau_2\sqrt{k} \leq T \leq \sum_{k=1}^{K}\tau_1+\tau_2\sqrt{k} ,
	\end{align}
	so the number of episodes $K$ is bounded by $(\frac{T}{\tau_1+\tau_2})^{2/3}\leq K\leq 3(\frac{T}{\tau_2})^{2/3}$ .\\
	Thus, we have shown that with probability at least
	$1-\frac{7}{2}\delta$,
	\begin{align}
	\mathcal{R}_T& \leq CT^{2/3}\sqrt{\log\left(\frac{3(S+1)}{\delta}T\right)}+T_0\rho^*,
	\end{align}
	where $S=2I$, and
	\begin{align}\label{para:C}
	C&=3\sqrt{2}\left[\left(D+1+\left(1+\frac{D}{2}\right)L_1\right)M^{3/2}C_1+ \left(1+\frac{D}{2}\right)L_2M^{1/2}C_2\right]\tau_2^{1/3}\tau_1^{-1/2}\\
	& \quad + 3\tau_2^{-2/3}(\tau_1\rho^*+D)+(D+1) \sqrt{2 \ln(\frac{1}{\delta})}.
	\end{align}
	Here, $M$ is the number of Markov chain states, 	where $L_1=4M(\frac{1-\epsilon}{\epsilon})^2/\min\left\{\bm\mu_{\min},1-\bm\mu_{\max}\right\}$, $L_2=4M(1-\epsilon)^2/\epsilon^3+\sqrt{M}$, with $\epsilon=\min_{1\leq i,j\leq M}P_{i,j}$, and $\bm\mu_{\max}$, $\bm\mu_{\min}$ are the maximum and minimum element of the matrix $\bm\mu$ respectively.

	\subsection{Proof of Proposition~\ref{prop:Azuma}}
	\begin{proof}
		For each episode $k=1, 2, \cdots, K$, let $E_k$ be the index set of the $k$th exploration phase, and $E=\cup_{k=1}^K E_k$ be the set of all exploitation periods among the horizon $T$ and $v_k$ be the value function of the optimistic POMDP at the $k$th exploitation phase. For an arbitrary time $t$, let $n=\sum_{i=1}^t \1_{i\in E}$, which means the number of exploitation periods up to time $t$. Define a stochastic process $\{Z_{n}, n \geq 0\}$:
		\begin{align}
		&Z_0=0,\\
		&Z_{n}=\sum_{j=1}^{n} \mathbb{E}^{\pi}[v_{k_j}(b_{t_j+1}^{k_j})|\mathcal{F}_{t_j}]-v_{k_j}(b_{t_j+1}^{k_j}),
		\end{align}
		where $k_j=\{k:j \in E_k\}$ and $t_j=\min\{t:\sum_{i=1}^t\1_{i \in E}=j\}$
		mean the corresponding episode and period of $j$th exploitation, respectively.

		We first show that $\{Z_{n}, n \geq 0\}$ is a martingale.
		Note that $\E^{\pi}[|Z_{n}|]\leq \sum_{j=1}^{n} \text{span}(v_{k_j})\leq nD<TD<\infty$.
		It remains to show $\E^{\pi}[Z_{n}|\mathcal{F}_{n-1}]=Z_{n-1}$ holds, i.e., $\E^{\pi}[Z_{n}-Z_{n-1}|\mathcal{F}_{n-1}]=0$.
		Note that
		\begin{align}
		\E^{\pi}[Z_{n}-Z_{n-1}|\mathcal{F}_{n-1}]=\E^{\pi}[\mathbb{E}^{\pi}[v_{k_n}(b_{t_n+1}^{k_n})|\mathcal{F}_{t_n}]-v_{k_n}(b_{t_n+1}^{k_n})|\mathcal{F}_{n-1}]=0,
		\end{align}
		where the last equality is due to $n-1 \leq n \leq t_n$ then applying the tower property.

		Therefore, $\{Z_n, n \geq 0\}$ is a martingale for any given policy $\pi$.
		Moreover, by Proposition~\ref{prop:span-uni-bound}, we have
		\begin{align*}
		|Z_{n}-Z_{n-1}|=|\mathbb{E}^{\pi}[v_{k_n}(b_{t_n+1}^{k_n})|\mathcal{F}_{t_n}]-v_{k_n}(b_{t_n+1}^{k_n})| \leq \text{span}(v_{k_n}) \leq D.
		\end{align*}
		Thus, $\{Z_n, n \geq 0\}$ is a martingale with bounded difference.

		Let $\bar{N}=\sum_{i=1}^T \1_{i\in E_k}$ and apply the Azuma-Hoeffding inequality \cite{Azuma1967}, we have
		\begin{align}
		\mathbb{P}(Z_{\bar{N}}-Z_{0} \geq \epsilon) \leq \exp \left(\frac{-\epsilon^{2}}{2 \sum_{t=1}^{\bar{N}} D^{2}}\right).
		\end{align}

		Note that $\bar{N} \leq T$ and $Z_{\bar{N}}=\sum_{k=1}^{K}\sum_{t\in E_k}\mathbb{E}^\pi[v_k(b_{t+1})|\mathcal{F}_t]-v_k(b_{t+1})$. Thus, setting $\epsilon=D \sqrt{2T \ln(\frac{1}{\delta})}$, we can obtain
		\begin{align}
		\mathbb{P}\left(\sum_{k=1}^{K}\sum_{t\in E_k}\mathbb{E}^\pi[v_k(b_{t+1})|\mathcal{F}_t]-v_k(b_{t+1}) \geq D \sqrt{2T \ln(\frac{1}{\delta})}\right) \leq \delta.
		\end{align}
		Hence we have completed the proof.
	\end{proof}

\end{document}